%% file: modl_tsp.tex
\documentclass[a4paper]{assets/IEEEtran}
\include{packages}

\ifCLASSOPTIONcompsoc%
 \usepackage[caption=false,font=normalsize,labelfont=sf,textfont=sf]{subfig}
\else
 \usepackage[caption=false,font=footnotesize]{subfig}
\fi

\include{commands}

\begin{document}
\input{ieee_title}
\input{main}

\appendices%
\input{appendices}

\bibliographystyle{assets/IEEEtran}
\scriptsize
\bibliography{extracted}
\normalsize

\vspace*{-2\baselineskip}
\begin{IEEEbiography}%
    [{\includegraphics[width=1in,height=1.25in,clip,keepaspectratio]{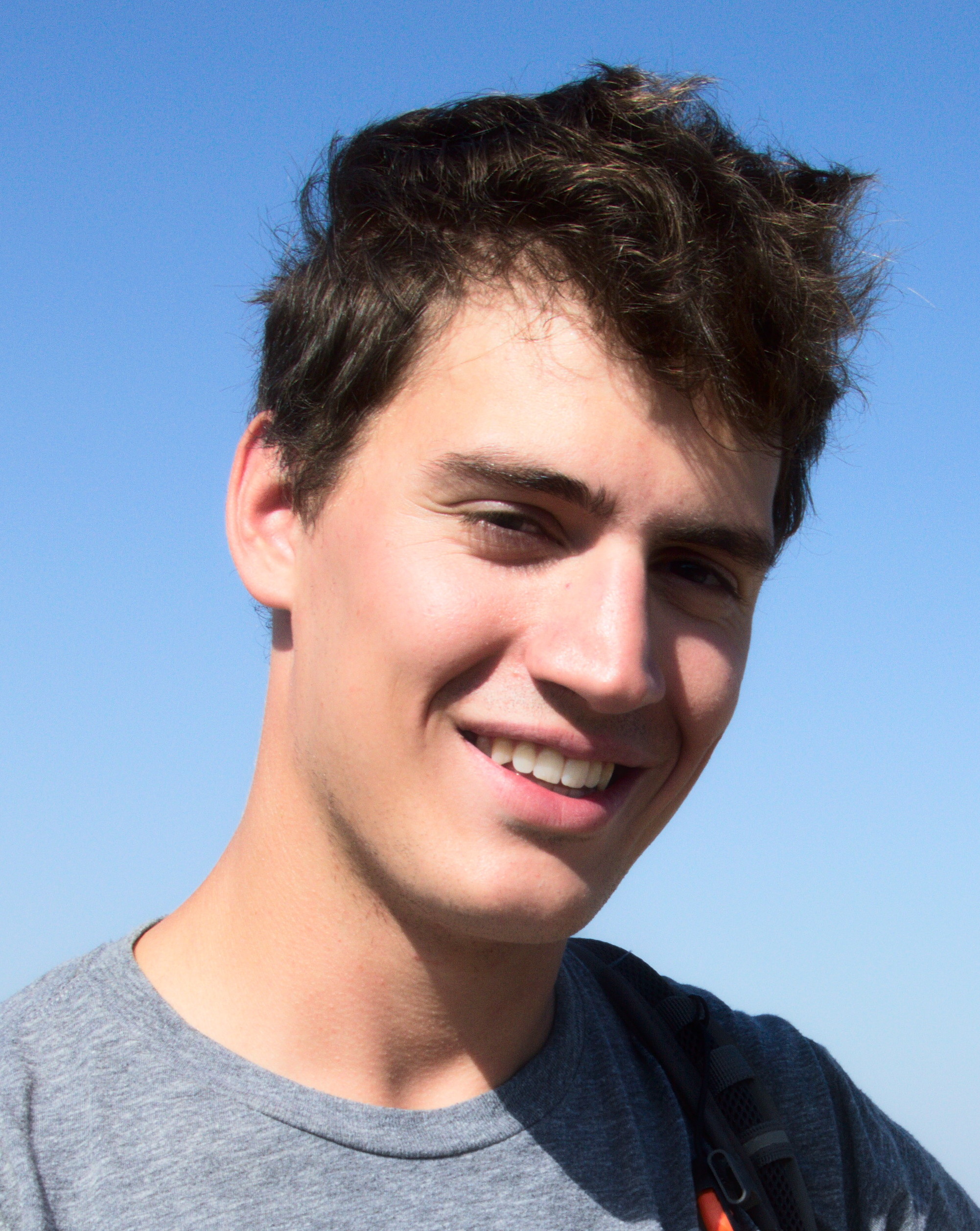}}]%
    {Arthur Mensch}is a PhD candidate at Universit\'e Paris-Saclay and Inria. His main research interests are related to large-scale stochastic optimization and statistical learning, with specific applications to functional neuroimaging and cognitive brain mapping. In 2015, he received a graduate degree from Ecole
    Polytechnique, France, and a MSc degree in applied mathematics from \'Ecole Normale
    Sup\'erieure de Cachan, France.
\end{IEEEbiography}
\vspace*{-2\baselineskip}
\begin{IEEEbiography}%
    [{\includegraphics[width=1in,height=1.25in,clip,keepaspectratio]{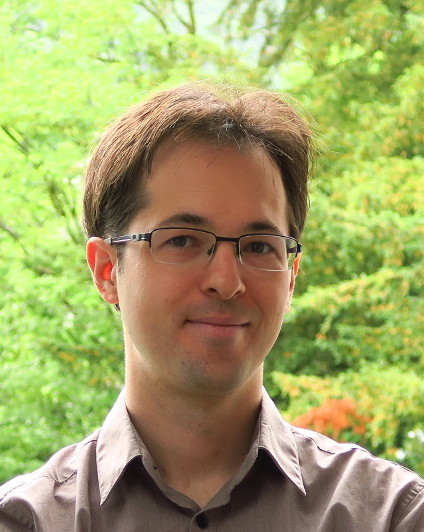}}]%
    {Julien Mairal} is a research scientist at
Inria. He received a graduate degree from Ecole
Polytechnique, France, in 2005, and a PhD degree
from Ecole Normale Supérieure, Cachan, France, in
2010. Then, he was a postdoctoral researcher at the
statistics department of UC Berkeley, before joining
Inria in 2012. His research interests include machine
learning, computer vision, mathematical optimiza-
tion, and statistical image and signal processing.
In 2016, he
received a Starting Grant from the European Research Council (ERC).
\end{IEEEbiography}
\vspace*{-2\baselineskip}
\begin{IEEEbiography}
[{\includegraphics[width=1in,height=1.25in,clip,keepaspectratio]{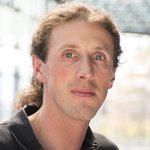}}]
{Ga\"el Varoquaux} is a tenured computer-science researcher at Inria. His
research develops statistical-learning tools for functional neuroimaging
data with application to cognitive mapping of the brain as well as the
study of brain pathologies. He is also heavily invested in
software development for data science, as project-lead for scikit-learn,
one of the reference machine-learning toolboxes, and on joblib, Mayavi,
and nilearn. Varoquaux has a PhD in quantum physics and is a graduate
from Ecole Normale Superieure, Paris.
\end{IEEEbiography}
\vspace*{-2\baselineskip}
\begin{IEEEbiography}
[{\includegraphics[width=1in,height=1.25in,clip,keepaspectratio]{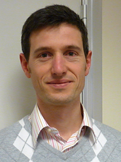}}]
{Bertrand Thirion} is the principal investigator of the Parietal team
(Inria-CEA) within the main French Neuroimaging center, Neurospin.
His main research interests are related to the use of machine learning and
statistical analysis techniques for neuroimaging, e.g. the modeling of brain variability in group studies, the mathematical study of functional
connectivity and brain activity decoding; he addresses various applications
such as the study of vision through neuroimaging and the classification of brain
images for diagnosis or brain mapping

\end{IEEEbiography}

\vfill
\newpage
\clearpage

\onecolumn

\input{derivations}
\end{document}

%% file: packages.tex
\usepackage[pdftex]{graphicx}
\graphicspath{{./figures/}}

\usepackage{amsmath}
\usepackage{amssymb}
\usepackage{amsfonts}
\usepackage{nicefrac}
\usepackage{cases}
\usepackage{algorithm}
\usepackage[noend]{algpseudocode}

\usepackage{booktabs}

\usepackage[square,numbers]{natbib}

\usepackage[shortcuts,cyremdash]{extdash}
\usepackage{xspace}

\usepackage{stfloats}

\usepackage{url}
\usepackage{hyperref}

\usepackage{amsthm}

\usepackage{enumerate}
\usepackage{color}
\usepackage[super]{nth}

\usepackage{etoolbox}

%% file: commands.tex
\def\EE{{\mathbb E}}

\def\RR{{\mathbb R}}
\def\PP{{\mathbb P}}

\def\x{{\mathbf x}}

\def\c{{\mathbf c}}

\def\A{{\mathbf A}}
\def\B{{\mathbf B}}
\def\C{{\mathbf C}}
\def\G{{\mathbf G}}

\def\D{{\mathbf D}}

\def\I{{\mathbf I}}

\def\M{{\mathbf M}}

\def\X{{\mathbf X}}
\def\Y{{\mathbf Y}}
\def\Z{{\mathbf Z}}

\def\b{{\mathbf b}}

\def\d{{\mathbf d}}

\def\u{{\mathbf u}}
\def\x{{\mathbf x}}

\newcommand{\balpha}{\boldsymbol{\alpha}}
\newcommand{\bbeta}{\boldsymbol{\beta}}
\newcommand{\bkappa}{\boldsymbol{\kappa}}

\renewcommand\P{{\mathbf P}}

\def\trace{{\mathrm{Tr}\;}}

\def\bigO{{\mathcal{O}}}

\newcommand{\argmin}{\operatornamewithlimits{\mathrm{argmin}}}

\newcommand{\vertiii}[1]{{\left\vert\kern-0.25ex\left\vert\kern-0.25ex\left\vert #1
    \right\vert\kern-0.25ex\right\vert\kern-0.25ex\right\vert}}

\newtheoremstyle{ieee}
  {}
  {}
  {}
  {}
  {\itshape}
  {.}
  { }
  {\thmname{#1}\thmnumber{ #2}\thmnote{ (#3)}}%

\makeatletter
\newtheoremstyle{assumption}
{}
{}
{\addtolength{\@totalleftmargin}{.0em}
   \addtolength{\linewidth}{-.0em}
   \parshape 1 0em \linewidth}
{1em}
{\itshape}
{}
{ }
{\thmname{#1}\thmnumber{#2} \textnormal{\textit{\thmnote{#3}}}}%
\makeatother

\theoremstyle{ieee}
\newtheorem{theorem}{Theorem}

\theoremstyle{assumption}
\newtheorem{assumption}{}
\renewcommand{\theassumption}{\textbf{(\Alph{assumption})}}

\theoremstyle{ieee}
\newtheorem{definition}{Definition}

\theoremstyle{ieee}
\newtheorem{proposition}[theorem]{Proposition}

\theoremstyle{ieee}

\theoremstyle{ieee}
\newtheorem{lemma}{Lemma}
\newtheorem*{lemma*}{Lemma}

\newcounter{tmp}
\newcounter{stoass}

\definecolor{dred}{rgb}{0.8,0,0}
\definecolor{dgreen}{rgb}{0,0.8,0}
\definecolor{dblue}{rgb}{0,0,0.8}
\definecolor{dpurple}{rgb}{0.8,0,0.8}

\algnewcommand\algorithmicswitch{\textbf{switch}}
\algnewcommand\algorithmiccase{\textbf{case}}
\algnewcommand\algorithmicassert{\texttt{assert}}
\algnewcommand\Assert[1]{\State \algorithmicassert(#1)}%
\algdef{SE}[SWITCH]{Switch}{EndSwitch}[1]{switch #1}{\algorithmicend\ \algorithmicswitch}%
\algdef{SE}[SWITCH]{Block}{EndBlock}[1]{#1}{}%
\algdef{SE}[CASE]{Case}{EndCase}[1]{\algorithmiccase\ #1}{\algorithmicend\ \algorithmiccase}%
\algtext*{EndSwitch}%
\algtext*{EndCase}%

\algnewcommand{\LineComment}[1]{\State \(\triangleright\) #1}
\algnewcommand{\Input}[1]{\State \textbf{Input}: #1}
\algnewcommand{\Output}[1]{\State \textbf{Output}: #1}

\newcommand{\somf}{\textsc{somf}\xspace}
\newcommand{\omf}{\textsc{omf}\xspace}
\newcommand{\smm}{\textsc{smm}\xspace}
\newcommand{\samm}{\textsc{samm}\xspace}

\newcommand*{\blue}{}
\DeclareRobustCommand{\blue}[1]{%
  \ifmmode
    \begingroup
      \color{blue}%
      #1%
    \endgroup
  \else
    \textcolor{blue}{#1}%
  \fi
}

\newcommand{\nocontentsline}[3]{}
\newcommand{\tocless}[2]{\bgroup\let\addcontentsline=\nocontentsline#1{#2}\egroup}

\newcommand{\specialcell}[2][c]{%
  \begin{tabular}[#1]{@{}c@{}}#2\end{tabular}}

%% file: ieee_title.tex

\title{Stochastic Subsampling for\\Factorizing Huge Matrices}

\author{Arthur~Mensch,
        Julien~Mairal,\\
        Bertrand~Thirion,
        and~Ga\"el~Varoquaux%

\thanks{A.~Mensch, B.~Thirion, G.~Varoquaux are with Parietal team, Inria, CEA,
Paris-Saclay University, Neurospin, at Gif-sur-Yvette, France. J.~Mairal is with Universit\'e Grenoble Alpes, Inria, CNRS, Grenoble INP, LJK at Grenoble, France.

 The research leading to these results was supported by the ANR (MACARON project, ANR-14-CE23-0003-01 --- NiConnect project, ANR-11-BINF-0004NiConnect).
 It has received funding from the European Union's Horizon 2020 Framework Programme for Research and Innovation under Grant Agreement N\textsuperscript{o}\xspace720270 (Human Brain Project SGA1).

 Corresponding author: Arthur Mensch (\url{arthur.mensch@m4x.org})
}
}

\markboth{Stochastic Subsampling for Factorizing Huge Matrices}%
{}

\maketitle

\input{abstract}


\IEEEpeerreviewmaketitle

%% file: abstract.tex

\begin{abstract}
We present a matrix-factorization algorithm that scales to input matrices with both huge number of rows and columns. Learned factors may be sparse or dense and/or non-negative, which makes our algorithm suitable for dictionary learning, sparse component analysis, and non-negative matrix factorization. Our algorithm streams matrix columns while subsampling them to iteratively learn the matrix factors. At each iteration, the row dimension of a new sample is reduced by subsampling, resulting in lower time complexity compared to a simple streaming algorithm. Our method comes with convergence guarantees to reach a stationary point of the matrix-factorization problem. We demonstrate its efficiency on massive functional Magnetic Resonance Imaging data (2 TB), and on patches extracted from hyperspectral images (103 GB). For both problems, which involve different penalties on rows and columns, we obtain significant speed-ups compared to state-of-the-art algorithms.
\end{abstract}

%% file: main.tex


\section{Introduction}\label{introduction}

Matrix factorization is a flexible approach to uncover latent factors in
low-rank or sparse models. With sparse factors, it is used in dictionary
learning, and has proven very effective for denoising and visual feature
encoding in signal and computer vision~\citep[see e.g.,][]{mairal_sparse_2014}. When the data admit a low-rank structure, matrix
factorization has proven very powerful for various tasks such as matrix completion
\citep{srebro_maximum-margin_2004,candes_exact_2009}, word embedding
\cite{pennington_glove:_2014,levy_neural_2014}, or network models \cite{zhang_spatio-temporal_2009}. It is flexible
enough to accommodate a large set of constraints and regularizations, and has
gained significant attention in scientific domains where interpretability is a
key aspect, such as genetics~\cite{kim_sparse_2007} and
neuroscience~\citep{varoquaux_multi-subject_2011}. In this paper, our goal
is to adapt matrix-factorization techniques to huge-dimensional datasets, i.e., with large number of columns~$n$ and large number of
rows~$p$. Specifically, our work is motivated by the rapid
increase in sensor resolution, as in hyperspectral imaging or fMRI,
and the challenge that the resulting high-dimensional signals pose
to current algorithms.

As a widely-used model, the literature on matrix factorization is very rich and
two main classes of formulations have emerged. The first one addresses a
convex-optimization problem with a penalty promoting low-rank structures,
such as the trace or max norms~\cite{srebro_maximum-margin_2004}. This
formulation has strong theoretical guarantees~\citep{candes_exact_2009},
but lacks scalability for huge datasets or sparse factors. For these reasons, our paper is focused on a second type of approach,
which relies on nonconvex optimization. Stochastic (or online) optimization
methods have been developed in this setting. Unlike classical alternate minimization procedures, they learn matrix decompositions by observing a single
matrix column (or row) at each iteration. In other words, they stream data
along one matrix dimension.  Their cost per iteration is significantly reduced,
leading to faster convergence in various practical contexts.  More precisely,
two approaches have been particularly successful: stochastic gradient descent
\citep{bottou_large-scale_2010} and stochastic majorization-minimization
methods \cite{mairal_stochastic_2013, razaviyayn_unified_2013}. The former has been widely used for
matrix completion \citep[see][and references therein]{burer_local_2004, recht_parallel_2013, bell_lessons_2007}, while
the latter has been used for dictionary learning with sparse and/or structured
regularization \cite{mairal_online_2010}. Despite those efforts, stochastic
algorithms for dictionary learning are currently unable to deal efficiently
with matrices that are large in both dimensions.

We propose a new matrix-factorization algorithm that can handle such matrices.
It builds upon the stochastic majorization-minimization framework
of \cite{mairal_stochastic_2013}, which we generalize for our problem. In
this framework, the objective function is minimized by iteratively improving an
upper-bound surrogate of the function (\textit{majorization} step) and
minimizing it to obtain new estimates (\textit{minimization} step). The core
idea of our algorithm is to approximate these steps to perform them faster.
We carefully introduce and control approximations, so to
extend convergence results of~\cite{mairal_stochastic_2013}  when neither the
majorization nor the minimization step is performed exactly.

For this purpose, we borrow ideas from \textit{randomized} methods in machine
learning and signal processing. Indeed, quite orthogonally to stochastic
optimization, efficient approaches to tackle the growth of dataset dimension
have exploited random projections
\citep{johnson_extensions_1984,bingham_random_2001} or sampling,
reducing data dimension while preserving signal content. Large-scale
datasets often have an intrinsic dimension which is significantly smaller than
their ambient dimension. Good examples are biological datasets
\cite{mckeown_analysis_1998} and physical acquisitions with an underlying
sparse structure enabling compressed sensing \cite{candes_near-optimal_2006}.
In this context, models can be learned using only random data summaries,
also called sketches. For instance, randomized methods \citep[see][for a
review]{halko_finding_2009} are efficient to compute PCA
\citep{rokhlin_randomized_2009}, a classic matrix-factorization approach,
and to solve constrained or penalized least-square problems
\cite{sarlos_improved_2006, lu_faster_2013}. On a theoretical level,
recent works on \textit{sketching} \cite{pilanci_iterative_2014,
raskutti_statistical_2014} have provided bounds on the risk of using random
summaries in learning.

Using random projections as a pre-processing step is not appealing in our
applicative context since factors learned on reduced data are not
interpretable. On the other hand, it is possible to exploit \textit{random
sampling} to approximate the steps of online matrix factorization. Factors
are learned in the original space whereas the dimension of each iteration is
reduced together with the computational cost per iteration.

\subsubsection*{Contribution}The contribution of this paper is both
practical and theoretical. We introduce a new matrix factorization algorithm,
called \textit{subsampled online matrix factorization} (\somf), which
is faster than state-of-the-art algorithms by an order of magnitude on large
real-world datasets (hyperspectral images, large fMRI data). It
leverages random sampling with stochastic optimization to learn sparse and
dense factors more efficiently. To prove the convergence of \somf, we
extend the stochastic majorization-minimization framework
\cite{mairal_stochastic_2013} and make it robust to some time-saving approximations.
We then show convergence guarantees for \somf under
reasonable assumptions. Finally, we propose an extensive empirical validation
of the subsampling approach.

In a first version of this work \cite{mensch_dictionary_2016} presented at the International Conference in
Machine Learning (\textsc{icml}), we proposed an
algorithm similar to \somf, without any theoretical guarantees. The algorithm
that we present here has such guarantees, which we express in a more general
framework, stochastic majorization-minimization. It is validated for new
sparsity settings and a new domain of application. An open-source efficient Python package
is provided.

\subsubsection*{Notations}Matrices are written using bold capital
letters and vectors using bold small letters (e.g.,~$\X$,~$\balpha$). We use
superscript to specify the column (sample or component) number, and write $\X =
[\x^{(1)},\ldots,\x^{(n)}]$. We use subscripts to specify the \textit{iteration}
number, as in $\x_t$. The floating bar, as in $\bar g_t$, is used to stress
that a given value is an average over iterations, or an expectation. The
superscript $^\star$ is used to denote an exact value, when it has to be compared
to an inexact value, e.g., to compare $\balpha_t^\star$ (exact) to
$\balpha_t$ (approximation).


\section{Prior art: matrix factorization with stochastic
         majorization-minimization}\label{sec:prior-art}

Below, we introduce the matrix-factorization problem and recall a
specific stochastic algorithm to solve it observing one column (or a mini-batch) at
every iteration. We cast this algorithm in the stochastic
majorization-minimization framework~\cite{mairal_stochastic_2013}, which we will use in the convergence analysis.

\subsection{Problem statement}\label{subset:problem-statement}
In our setting, the goal of matrix factorization is to decompose a matrix $\X
\in \RR^{p \times n}$ --- typically~$n$ signals of dimension~$p$ --- as a product
of two smaller matrices:
\begin{equation*}
    \X \approx \D \A
 \quad \text{with}\quad\D \in \RR^{p \times k}\:\text{and}\:\A \in \RR^{k \times n},
\end{equation*}
with potential sparsity or structure requirements on $\D$ and~$\A$. In
signal processing, sparsity is often enforced on the code~$\A$, in a
problem called \textit{dictionary learning} \cite{olshausen_sparse_1997}.
In such a case, the matrix~$\D$ is called the
``dictionary'' and~$\A$ the sparse code. We use this terminology throughout
the paper.

Learning the factorization is typically performed by minimizing a quadratic
data-fitting term, with constraints and/or penalties over the code and the
dictionary:
 \begin{equation}
     \label{eq:empirical-risk-factored}
    \min_{\substack{\D \in \mathcal{C} \\
	 \A \in \RR^{k\times n}}}\quad  \sum_{i=1}^n
    \frac{1}{2}
    \bigl\|
    \x^{(i)}
    - \D \balpha^{(i)}
    \bigr\|_2^2 + \lambda \, \Omega(\balpha^{(i)}),
 \end{equation}
where $\A \triangleq [\balpha^{(1)}, \ldots, \balpha^{(n)}]$, $\mathcal{C}$ is a column-wise separable convex set of $\RR^{p \times
k}$ and $\Omega : \RR^p \rightarrow \RR$ is a penalty over the code. Both
constraint set and penalty may enforce structure or sparsity, though
$\mathcal{C}$ has traditionally been used as a technical requirement to ensure
that the penalty on $\A$ does not vanish with $\D$ growing arbitrarily large.
Two choices of $\mathcal{C}$ and $\Omega$ are of particular interest. The
problem of dictionary learning sets $\mathcal{C}$ as the $\ell_2$ ball for each atom
and $\Omega$ to be the $\ell_1$ norm. Due to the sparsifying effect of $\ell_1$
penalty \cite{tibshirani_regression_1996}, the dataset admits a
\textit{sparse} representation in the dictionary. On the opposite, finding a \textit{sparse set} in which to
represent a given dataset, with a goal akin to sparse PCA \cite{zou_sparse_2006}, requires to set
as the $\ell_1$ ball for each atom
and $\Omega$ to be the $\ell_2$ norm. Our work considers the \textit{elastic-net} constraints
and penalties~\cite{zou_regularization_2005}, which encompass both special cases. Fixing $\nu$ and $\mu$ in $[0, 1]$, we
denote by $\Omega(\cdot)$ and $\Vert \cdot \Vert$ the elastic-net penalty in $\RR^p$ and $\RR^k$:
\begin{gather}
    \label{eq:elastic_net}
        \Omega(\balpha) \triangleq (1 - \nu) \Vert \balpha \Vert_1
        + \frac{\nu}{2} \Vert \balpha \Vert_2^2,  \\
        \mathcal{C} \triangleq \left\{ \D \in \RR^{p\times k}/
        \Vert \d^{(j)} \Vert
         \triangleq (1 {-} \mu) \Vert \d^{(j)} \Vert_1
         {+} \frac{\mu}{2} \Vert \d^{(j)} \Vert_2^2 \leq 1 \right\}. \notag
\end{gather}
Following~\cite{mairal_online_2010}, we can also enforce the positivity of~$\D$ and/or~$\A$ by replacing~$\RR$
by $\RR^+$ in $\mathcal{C}$, and adding positivity constraints on $\A$ in~\eqref{eq:empirical-risk-factored},
 as in non-negative sparse coding~\cite{hoyer_non-negative_2002}.
We rewrite \eqref{eq:empirical-risk-factored} as an empirical risk minimization
problem depending on the dictionary only. The matrix $\D$
solution of~\eqref{eq:empirical-risk-factored} is indeed obtained by
minimizing the empirical risk $\bar f$
\begin{align}
    \label{eq:empirical-risk}
    \D \in \argmin_{\D \in \mathcal{C}}
    \Big(\bar f(\D) \triangleq \frac{1}{n} \sum_{i=1}^n f(\D, \x^{(i)}) \Big), \\
    \text{where}\quad
     f(\D, \x) \triangleq
      \min_{\balpha \in \RR^k} \frac{1}{2}
     \bigl\|
     \x
     - \D \balpha
     \bigr\|_2^2 + \lambda \, \Omega(\balpha), \notag
\end{align}
and the matrix $\A$ is obtained by solving the linear regression
\begin{equation}
 \min_{\A \in \RR^{k\times n}} \sum_{i=1}^n
\frac{1}{2}
\bigl\|
\x^{(i)}
- \D \balpha^{(i)}
\bigr\|_2^2 + \lambda \, \Omega(\balpha^{(i)}).
\end{equation}
The problem~\eqref{eq:empirical-risk-factored} is non-convex in the parameters
$(\D, \A)$, and hence \eqref{eq:empirical-risk} is not convex. However,
the problem~\eqref{eq:empirical-risk-factored} is convex in both $\D$ and $\A$
when fixing one variable and optimizing with respect to the other. As such, it is
naturally solved by alternate minimization over $\D$ and $\A$, which asymptotically provides
a stationary point of~\eqref{eq:empirical-risk}. Yet, $\X$ has typically to be observed
hundred of times before obtaining a good dictionary. Alternate minimization is
therefore not adapted to datasets with many samples.

\subsection{Online matrix factorization}

When $\X$ has a large number of columns but a limited number of
rows, the stochastic optimization method of \cite{mairal_online_2010}
outputs a good dictionary much more rapidly than alternate-minimization.
In this setting \citep[see][]{bottou_optimization_2016}, learning the dictionary is
naturally formalized as an expected risk minimization
\begin{gather}\label{eq:expected-risk}
    \min_{\D \in \mathcal{C}}\quad \bar f(\D) \triangleq \EE_\x[f(\D, \x)],
\end{gather}
where $\x$ is drawn from the data distribution and forms an i.i.d. stream
$(\x_t)_t$. In the finite-sample setting,~\eqref{eq:expected-risk} reduces to
\eqref{eq:empirical-risk} when $\x_t$ is drawn uniformly at random from
$\{\x^{(i)}, i \in [1, n]\}$. We then write $i_t$ the sample number selected at
time $t$.

The online matrix factorization algorithm proposed in \cite{mairal_online_2010}
is summarized in Alg.~\ref{alg:omf}.
It draws a sample $\x_t$ at each iteration, and uses it to improve the current
iterate $\D_{t-1}$. For this, it first computes the code~$\balpha_t$ associated
to~$\x_t$ on the current dictionary:
\begin{equation}\label{eq:code_computation}
    \balpha_t \triangleq \argmin_{\balpha \in \RR^k} \frac{1}{2} \Vert \x_t - \D_{t-1} \balpha \Vert_2^2 + \lambda \Omega(\balpha).
\end{equation}
Then, it updates $\D_t$ to make it optimal in reconstructing past
samples $(\x_s)_{s \leq t}$ from previously computed codes $(\balpha_s)_{s \leq
t}$:
\begin{equation}
    \label{eq:bar_gt}
    \D_t \in \argmin_{\D \in \mathcal{C}}\Bigl( \bar g_t(\D)\triangleq
    \frac{1}{t}\sum_{s=1}^t \frac{1}{2}
    \bigl\|
    \x_s
    - \D \balpha_s
    \bigr\|_2^2  + \lambda \Omega(\balpha_s) \Bigr).
\end{equation}
\begin{algorithm}[t]
\begin{algorithmic}
    \Input Initial iterate $\D_0$,
    sample stream ${(\x_t)}_{t> 0}$,
    number of iterations $T$.
    \For{$t$ from $1$ to $T$}
    \State Draw $\x_t \sim \mathcal{P}$.
    \State Compute $\balpha_t = \argmin_{\balpha \in \RR^p} \frac{1}{2}
                \bigl\|
                \x_t
                - \D_{t-1} \balpha
                \bigr\|_2^2 + \lambda \, \Omega(\balpha)$.
    \State Update the \textit{parameters} of aggregated surrogate $\bar g_t$:
    \begin{equation}
        \label{eq:omf-parameter-update}
        \begin{split}
            \bar \C_t &= \Big(1 - \frac{1}{t}\Big) \bar \C_{t-1} + \frac{1}{t} \balpha_t \balpha_t^\top. \\
            \bar \B_t &= \Big(1 - \frac{1}{t}\Big) \bar \B_{t-1} + \frac{1}{t} \x_t \balpha_t^\top.
        \end{split}
    \end{equation}
    \State Compute (using block coordinate descent):
    \begin{equation*}
        \D_t = \argmin_{\D \in \mathcal{C}}
        \frac{1}{2} \trace (\D^\top \D \bar \C_t)
          - \trace (\D^\top \bar \B_t).
    \end{equation*}
    \EndFor
    \Output Final iterate $\D_T$.
\end{algorithmic}
\caption{Online matrix factorization (\omf)~\cite{mairal_online_2010}}\label{alg:omf}
\end{algorithm}
Importantly, minimizing $\bar g_t$ is equivalent to minimizing the quadratic function
\begin{align}
    \D \to \frac{1}{2} \trace (\D^\top \D \bar \C_t^\top)
      - \trace (\D^\top \bar \B_t),
\end{align}
where $\bar \B_t$ and $\bar \C_t$ are small matrices that summarize previously seen samples and codes:
\begin{equation}
    \label{eq:bar_BCt}
    \begin{split}
    \bar \B_t &= \frac{1}{t} \sum_{s=1}^t \x_s \balpha_s^\top \qquad
    \bar \C_t = \frac{1}{t} \sum_{s=1}^t \balpha_s \balpha_s^\top.
    \end{split}
\end{equation}
As the constraints $\mathcal{C}$ have a separable structure per atom,
\cite{mairal_online_2010} uses projected block coordinate descent to minimize $\bar g_t$.
The function gradient writes $\nabla \bar g_t(\D) = \D \bar \C_t - \bar \B_t$,
and it is therefore enough to maintain $\bar \B_t$ and $\bar \C_t$ in memory to
solve~\eqref{eq:bar_gt}. $\bar \B_t$ and $\bar \C_t$ are updated online,
using the rules~\eqref{eq:omf-parameter-update}~(Alg.~\ref{alg:omf}).

The function $\bar g_t$ is an upper-bound surrogate of the true current empirical risk, whose definition
involves the regression minima computed on current dictionary $\D$:
\begin{equation}
    \label{eq:bar_ft}
    \bar f_t(\D)\triangleq \frac{1}{t} \sum_{s=1}^t \min_{\balpha \in \RR^p} \frac{1}{2}
        \bigl\|
        \x_s
        - \D \balpha
        \bigr\|_2^2\!+ \lambda \Omega(\balpha)\leq \bar g_t(\D).
\end{equation}
Using empirical processes theory \cite{van_der_vaart_asymptotic_2000},
it is possible to
show that minimizing $\bar f_t$ at each iteration asymptotically
yields a stationary point of the expected risk \eqref{eq:expected-risk}.
Unfortunately, minimizing \eqref{eq:bar_ft} is expensive as it involves the
computation of optimal current codes for every previously seen sample at each
iteration, which boils down to naive alternate-minimization.

\begin{algorithm}[t]
\begin{algorithmic}
    \Input Initial iterate $\theta_0$, weight sequence ${(w_t)}_{t>0}$,
    sample stream ${(\x_t)}_{t> 0}$, number of iteration $T$.
    \For{$t$ from $1$ to $T$}
    \State Draw $x_t \sim \mathcal{P}$, get $f_t : \theta \in \Theta \to f(\x_t, \theta)$.
    \State Construct a surrogate of $f_t$ near $\theta_{t-1}$, that meets
        \begin{equation}
			\label{eq:majorization_step_smm}
         g_t \geq f_t,\quad g_t(\theta_{t-1}) = f_t(\theta_{t-1}).
        \end{equation}
    \State Update the aggregated surrogate:
    \begin{equation*}
        \bar g_t = (1 - w_t) \bar g_{t-1} + w_t g_t.
    \end{equation*}
    \State Compute
        \begin{equation}
			\label{eq:minimization_step_smm}
            \theta_t = \argmin_{\theta \in \Theta} \bar g_t(\theta).
        \end{equation}
    \EndFor
    \Output Final iterate $\theta_T$.
\end{algorithmic}
\caption{Stochastic majorization-minimization \citep[\smm][]{mairal_stochastic_2013}}\label{alg:smm}
\end{algorithm}

In contrast, $\bar g_t$ is much cheaper to minimize than $\bar f_t$, using
block coordinate descent. It is possible to show that $\bar g_t$ converges
towards a locally tight upper-bound of the objective~$\bar f_t$ and that
minimizing $\bar g_t$ at each iteration also asymptotically yields a stationary
point of the expected risk~\eqref{eq:expected-risk}. This establishes the
correctness of the \textit{online matrix factorization} algorithm (\omf). In practice, the \omf algorithm performs a
single pass of block coordinate descent: the minimization step is inexact. This
heuristic will be justified by our theoretical contribution in Section~\ref{sec:analysis}.

\subsubsection*{Extensions}For efficiency, it is essential to use mini-batches $\{\x_s,\,s
\in \mathcal{T}_t\}$ of size $\eta$ instead of single samples in the
iterations~\cite{mairal_online_2010}. The surrogate parameters $\bar \B_t$, $\bar \C_t$ are
then updated by the mean value of $\{(\x_s \balpha_s^\top, \balpha_s \balpha_s^\top)\}_{s \in \mathcal{T}_t}$
over the batch. The optimal size of the mini-batches is usually close to $k$. \eqref{eq:omf-parameter-update}
uses the sequence of weights ${(\frac{1}{t})}_t$ to update
parameters~$\bar \B_t$ and~$\bar \C_t$. \cite{mairal_online_2010}
replaces these weights with a sequence ${(w_t)}_t$,
which can decay more slowly to give more importance to recent samples in $\bar g_t$. These weights will prove important in our analysis.

\subsection{Stochastic majorization-minimization} Online matrix factorization
belongs to a wider category of algorithms introduced in
\cite{mairal_stochastic_2013} that minimize locally tight upper-bounding
surrogates instead of a more complex objective, in order to solve an expected
risk minimization problem. Generalizing online matrix factorization, we
introduce in Alg.~\ref{alg:smm} the \textit{stochastic majorization-minimization} (\smm) algorithm,
which is at the core of our theoretical contribution.

In online matrix factorization, the true empirical risk functions $\bar f_t$ and their
\textit{surrogates} $\bar g_t$ follow the update rules, with generalized weight ${(w_t)}_t$
set to ${(\frac{1}{t})}_t$ in \eqref{eq:bar_gt} -- \eqref{eq:bar_ft}:
\begin{equation}
    \label{eq:surrogate-aggregation}
        \bar f_t \triangleq (1 - w_t) \bar f_{t-1} + w_t f_t,\quad
        \bar g_t \triangleq (1 - w_t) \bar g_{t-1} + w_t g_t,
\end{equation}
where the \textit{pointwise} loss function and its surrogate are
\begin{equation}
	\begin{split}
    \label{eq:surrogate}
    f_t(\D) &\triangleq \min_{\balpha \in \RR^k} \frac{1}{2} \Vert \x_t - \D \balpha \Vert_2^2
	+ \lambda \Omega(\balpha),
	 \\
	 g_t(\D) &\triangleq \frac{1}{2} \Vert \x_t - \D \balpha_t \Vert_2^2 + \lambda \Omega(\balpha_t).
 	\end{split}
\end{equation}
The function $g_t$ is a majorizing surrogate of $f_t$: $g_t
\geq f_t$, and $g_t$ is tangent to $f_t$ in $\D_{t-1}$, i.e, $g_t(\D_{t-1}) = f_t(\D_{t-1})$ and $\nabla(g_t - f_t)(\D_{t-1}) =
0$.
At each step of online matrix factorization:
\begin{itemize}
    \item The surrogate $g_t$ is computed along with $\balpha_t$, using~\eqref{eq:code_computation}.
    \item The parameters $\bar \B_t, \bar \C_t$ are updated following~\eqref{eq:omf-parameter-update}.
    They define the \textit{aggregated} surrogate $\bar g_t$ up to a constant.
    \item The quadratic function $\bar g_t$ is minimized efficiently by block coordinate descent,
    using parameters $\bar \B_t$ and $\bar \C_t$ to compute its gradient.
\end{itemize}

The stochastic majorization-minimization framework simply formalizes the three
steps above, for a larger variety of loss functions $f_t(\theta) \triangleq
f(\theta, \x_t)$, where $\theta$ is the parameter we want to learn ($\D$ in
the online matrix factorization setting). At iteration $t$, a surrogate $g_t$
of the loss $f_t$ is computed to update the aggregated surrogate $\bar g_t$ following \eqref{eq:surrogate-aggregation}.
 The surrogate functions $(g_t)_t$ should be upper-bounds of
loss functions $(f_t)_t$, tight in the current iterate $\theta_{t-1}$ (e.g.,
the dictionary $\D_{t-1}$). This simply means that $f_t(\theta_{t-1}) =
g_t(\theta_{t-1})$ and $\nabla(f_t - g_t)(\theta_{t-1}) = 0$. Computing $\bar g_t$ can be done
if $g_t$ is defined simply, as in \omf where it is linearly parametrized by
$(\balpha_t \balpha_t^\top, \x_t \balpha_t^\top)$. $\bar g_t$ is then minimized
to obtain a new iterate~$\theta_t$.

It can be shown following~\cite{mairal_stochastic_2013} that stochastic majorization-minimization
algorithms find asymptotical stationary point of the expected risk
$\EE_\x[f(\theta, \x)]$ under mild assumptions recalled in
Section~\ref{sec:analysis}. \smm admits the same mini-batch and decaying weight
extensions (used in Alg.~\ref{alg:smm}) as \omf.

In this work, we extend the \smm framework and allow both majorization and
minimization steps to be approximated. As a side contribution, our extension
proves that performing a single pass of block coordinate descent to update the dictionary,
an important heuristic in~\cite{mairal_online_2010}, is indeed
correct. We first introduce the new matrix factorization algorithm at the
core of this paper and then present the extended \smm framework.

\section{Stochastic
subsampling for high dimensional data decomposition}\label{sec:stochastic-subsampling}

The online algorithm presented in Section~\ref{sec:prior-art} is very
efficient to factorize matrices that have a large number of columns
(i.e., with a large number of samples $n$), but a reasonable number of
rows --- the dataset is not very high dimensional. However, it is not designed
to deal with very high number of rows: the cost of a single iteration
depends linearly on $p$. On terabyte-scale datasets from fMRI with
$p=2\cdot10^5$ features, the original online algorithm requires one week
to reach convergence. This is a major motivation for designing new matrix
factorization algorithms that scale in \textit{both directions}.

In the large-sample regime $p \gg k$, the underlying dimensionality of columns
may be much lower than the actual $p$: the rows of a single column drawn at random
are therefore correlated and redundant. This guides us on how
to scale online matrix factorization with regard to the number of rows:
\begin{itemize}
    \item The online algorithm \omf uses a \textit{single} column of (or
    mini-batch) of $\X$ at each iteration to enrich the average surrogate and update
    the \textit{whole} dictionary.
    \item We go a step beyond and use a \textit{fraction} of
    a single column of $\X$ to refine a fraction of the dictionary.

\end{itemize}
More precisely, we draw a column and observe only \textit{some}
of its rows at each iteration, to refine these rows of the
dictionary, as illustrated in Figure~\ref{fig:next-level}. To take into account
all features from the dataset, rows are selected at random at each
iteration: we call this technique \textit{stochastic subsampling}.
Stochastic subsampling reduces the efficiency of the dictionary update \textit{per
iteration}, as less information is incorporated in the current iterate $\D_t$.
On the other hand, with a correct design, the cost of a single
iteration can be considerably reduced, as it grows with the number of observed
features. Section~\ref{sec:experiments} shows that the proposed
algorithm is an order of magnitude faster than the original \omf on large
and redundant datasets.

\begin{figure}
    \centering
    \includegraphics{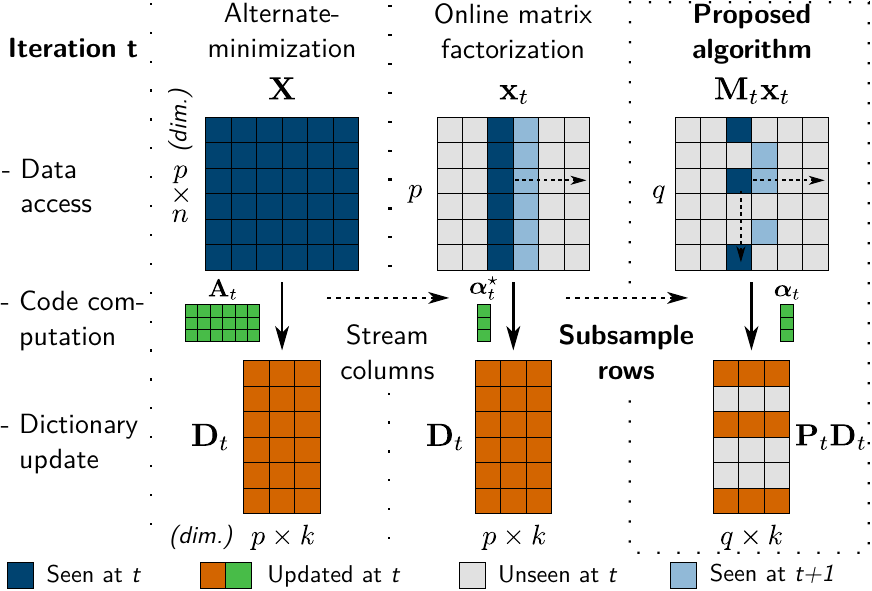}
    \caption{\textit{Stochastic subsampling} further improves online
    matrix factorization to handle datasets with large number of columns and
    rows. $\X$ is the input~$p\times n$ matrix, $\D_t$ and $\A_t$ are respectively the dictionary and code at time~$t$.}\label{fig:next-level}
    \vspace{-.5em}
\end{figure}

First, we formalize the idea of working with a fraction of the $p$ rows at a
single iteration. We adapt the online matrix factorization algorithm, to reduce
the iteration cost by a factor close to the ratio of selected rows. This
defines a new online algorithm, called \textit{subsampled online matrix
factorization} (\somf). At each iteration, it uses $q$ rows of
the column $\x_t$ to update the sequence of iterates ${(\D_t)}_t$. As in
Section~\ref{sec:prior-art}, we introduce a more general algorithm,
\textit{stochastic approximate majorization-minimization} (\samm), of which
\somf is an instance. It extends the stochastic majorization-minimization
framework, with similar theoretical guarantees but potentially faster
convergence.

\subsection{Subsampled online matrix factorization}
\label{sec:somf}

Formally, as in online matrix factorization, we consider a sample stream
${(\x_t)}_t$ in $\RR^p$ that cycles onto a finite sample set $\{ \x^{(i)}, i \in [1, n] \}$,
and minimize the empirical risk~\eqref{eq:empirical-risk}.\footnote{Note that we solve the fully observed
problem despite the use of subsampled data, unlike other recent work on low-rank factorization~\cite{mardani_subspace_2015}.}

\subsubsection{Stochastic subsampling and algorithm outline}
We want to reduce the time complexity of a single iteration.
In the original algorithm, the complexity
depends linearly on the sample dimension $p$ in three aspects:
\begin{itemize}
    \item $\x_t \in \RR^p$ is
    used to compute the code $\balpha_t$,
    \item it is used to update the surrogate parameters
    $\bar \B_t \in \RR^{p\times k}$,
    \item $\D_t \in \RR^{p\times k}$ is fully updated at each iteration.
\end{itemize}
Our algorithm reduces the dimensionality of
these steps at each iteration, such that $p$ becomes $q =
\frac{p}{r}$ in the time complexity analysis, where $r > 1$ is a
\textit{reduction factor}. Formally, we randomly draw, at iteration $t$, a mask $\M_t$ that ``selects'' a
random subset of $\x_t$. We use it to drop a part of the features of~$\x_t$ and
to ``freeze'' these features in dictionary~$\D$ at iteration~$t$.

It is convenient to consider $\M_t$ as a $\RR^{p\times
p}$ random diagonal matrix, such that each coefficient is a Bernouilli variable
with parameter $\frac{1}{r}$, normalized to be $1$ in expectation. $\forall j \in [0, p-1]$,
\begin{equation}
    \label{eq:mt}
    \PP\bigl[\M_t[j, j] = r\bigr] =
    \frac{1}{r},\;\;\;\PP\bigl[\M_t[j, j] = 0\bigr] = 1 - \frac{1}{r}.
\end{equation}
Thus, $r$ describes the average proportion of observed features and $\M_t
\x_t$ is a non-biased, low-dimensional estimator of $\x_t$:
\begin{equation}
    \EE\bigl[\Vert \M_t \x_t \Vert_0\bigr] = \frac{p}{r} = q \qquad
    \EE\bigl[\M_t \x_t\bigr] = \x_t.
\end{equation}
with $\Vert \cdot \Vert_0$ counting the number of non-zero coefficients. We define the pair of orthogonal projectors $\P_t \in \RR^{q \times p}$ and
$\P_t^\perp \in \RR^{(p - q)\times p}$ that project $\RR^p$ onto $\mathrm{Im}(\M_t)$ and $\mathrm{Ker}(\M_t)$.
In other words, $\P_t \Y$ and $\P_t^\perp \Y$ are the submatrices of $\Y \in \RR^{p \times y}$ with rows respectively selected and not selected by $\M_t$.
In algorithms, $\P_t \Y \leftarrow \Z \in \RR^{q \times n}$ assigns
 the rows of $\Z$ to the rows of $\Y$ selected by $\P_t$, by an abuse of notation.

\begin{algorithm}[t]
    \begin{algorithmic}
        \Input Initial iterate $\D_0$, weight sequences ${(w_t)}_{t>0}$, ${(\gamma_c)}_{c>0}$,
        sample set ${\{\x^{(i)}\}}_{i> 0}$, number of iterations $T$.
        \For{$t$ from $1$ to $T$}
        \State Draw $\x_t = \x^{(i)}$ at random and $\M_t$ following~\eqref{eq:mt}.
        \State Update the regression parameters for sample $i$:
            \begin{align*}
            c^{(i)} &\gets c^{(i)} + 1,  &\gamma &\gets \gamma_{c^{(i)}}.\\
             \bbeta_t^{(i)} &\gets (1 - \gamma) \G_{t-1}^{(i)} + \gamma \D_{t-1}^\top \M_t \x^{(i)}, &   \bbeta_t &\gets  \bbeta_t^{(i)}. \\
             \G_t^{(i)} &\gets (1 - \gamma) \G_{t-1}^{(i)} +\gamma \D_{t-1}^\top \M_t \D_{t-1}, &  \G_t &\gets \bar \G_t^{(i)}.
            \end{align*}
        \State Compute the approximate code for $\x_t$:
        \begin{equation}\label{eq:somf_alpha}
             \balpha_t \gets \argmin_{\balpha \in \RR^k}
            \frac{1}{2} \balpha^\top  \G_t \balpha -
            \balpha^\top  \bbeta_t + \lambda \, \Omega(\balpha).
        \end{equation}
        \State Update the parameters of the aggregated surrogate $\bar g_t$:
        \begin{equation}
            \label{eq:somf_partial}
            \begin{split}
             \bar \C_t &\gets (1 - w_t) \bar \C_{t-1} + w_t  \balpha_t  \balpha_t^\top. \\
             \P_t \bar \B_t &\gets (1 - w_t ) \P_t \bar \B_{t-1} + w_t \P_t \x_t  \balpha_t^\top.
            \end{split}
        \end{equation}
        \State Compute simultaneously (using Alg.~\ref{alg:somf-dictionary} for \nth{1} line):
        \begin{align}
            \label{eq:somf_minimization}
            \P_t \D_t &\gets \argmin_{\D^r \in \mathcal{C}^r}
            \frac{1}{2} \trace ({\D^r}^\top \D^r \bar \C_t)
              - \trace ({\D^r}^\top \P_t  \bar \B_t). \notag \\
             \P_t^\perp \bar \B_t &\gets (1 - w_t ) \P_t^\perp \bar \B_{t-1} + w_t \P_t^\perp \x_t  \balpha_t^\top.
        \end{align}
        \EndFor
        \Output Final iterate $\D_T$.
    \end{algorithmic}
    \caption{Subsampled online matrix factorization (\somf)}\label{alg:somf}
\end{algorithm}

In brief, subsampled online matrix factorization, defined in Alg.~\ref{alg:somf},
follows the outer loop of online matrix factorization, with the following major modifications at iteration $t$:
\begin{itemize}
    \item it uses $\M_t \x_t$ and low-size statistics instead of $\x_t$ to
    estimate the code $\balpha_t$ and the surrogate $g_t$,
    \item it updates a subset of the dictionary $\P_t \D_{t-1}$ to reduce
    the surrogate value $\bar g_t(\D)$. Relevant parameters of $\bar g_t$ are computed
    using $\P_t \x_t$ and $\balpha_t$ only.
\end{itemize}
We now present \somf in details. For comparison purpose, we write
all variables that would be computed following the \omf rules at iteration $t$
with a $^\star$ superscript. For simplicity, in Alg.~\ref{alg:somf} and in the following
paragraphs, we assume that we use one sample per iteration ---in
practice, we use
mini-batches of size $\eta$. The next derivations are
transposable when a batch $I_t$ is drawn at iteration $t$ instead of a single sample $i_t$.
\subsubsection{Code computation}\label{sec:code-computation}
In the \omf algorithm presented in Section~\ref{sec:prior-art},
$\balpha_t^\star$ is obtained by solving~\eqref{eq:code_computation}, namely
\begin{equation}
    \label{eq:regression}
    \balpha_t^\star \in \argmin_{\balpha} \frac{1}{2} \balpha^\top \G_t^\star \balpha - \balpha^\top \bbeta_t^\star + \lambda \Omega(\balpha),
\end{equation}
where $\G_t^\star = \D_{t-1}^\top \D_{t-1}$ and $\bbeta_t^\star = \D_{t-1}^\top \x_t$. For
large $p$, the computation of $\G_t^\star$ and $\bbeta_t^\star$ dominates the complexity of the
regression step, which depends almost linearly on $p$. To reduce this complexity, we use
\textit{estimators} for $\G_t^\star$ and $\bbeta_t^\star$, computed at a cost
proportional to the reduced dimension~$q$. We propose three kinds of
estimators
with different properties.

\paragraph{Masked loss}The most simple \textit{unbiased} estimation of~$\G_t^\star$ and~$\bbeta_t^\star$
whose computation cost depends on $q$ is obtained
by subsampling matrix products with $\M_t$:
\begin{equation}
    \label{eq:estimates}
    \tag{a}
    \begin{split}
    \G_t &= \D_{t-1}^\top \M_t \D_{t-1} \\
    \bbeta_t &= \D_{t-1}^\top \M_t \x_t.
    \end{split}
\end{equation}
This is the strategy proposed in~\cite{mensch_dictionary_2016}. We use $\G_t$ and $\bbeta_t$ in~\eqref{eq:somf_alpha},
which amounts to minimize the \textit{masked loss}
\begin{equation}
    \label{eq:masked-loss}
    \min_{\balpha \in \RR^k} \frac{1}{2} \Vert \M_t(\x^t - \D_{t-1}^\top \balpha) \Vert_2^2 + \lambda \Omega(\balpha).
\end{equation}
$\G_t$ and $ \bbeta_t$ are computed in a number of operations proportional to
$q$, which brings a speed-up factor of almost $r$ in the code computation for
large $p$. On large data, using estimators~\eqref{eq:estimates} instead of exact $\G_t^\star$ and $\bbeta_t^\star$
proves very efficient during the first epochs (cycles over the columns).%
\footnote{Estimators \eqref{eq:estimates} are also available in the infinite sample setting, when minimizing expected risk~\eqref{eq:expected-risk} from a i.i.d sample stream ${(\x_t)}_t$.}
However, due to the masking, $\G_t$ and $
\bbeta_t$ are not \textit{consistent} estimators: they do not
converge to $\G_t^\star$ and $\bbeta_t^\star$ for large~$t$, which breaks
theoretical guarantees on the algorithm output.
Empirical results in Section~\ref{sec:empirical_high_reduction} show that
the sequence of iterates approaches a
critical point of the risk~\eqref{eq:empirical-risk},
 but may then oscillate around it.

\paragraph{Averaging over epochs}At iteration $t$, the sample $\x_t$ is drawn from a finite set of samples ${\{\x^{(i)}\}}_i$.
This allows to
average estimators over previously seen samples
 and address the non-consistency issue of~\eqref{eq:estimates}. Namely, we keep in
memory $2n$ estimators, written ${( \G_t^{(i)}, \bbeta^{(i)}_t)}_{1\leq i \leq
n}$. We observe the sample $i = i_t$ at iteration~$t$ and use it to
update the $i$-th estimators $\bar \G_t^{(i)}$, $\bar \bbeta^{(i)}_t$ following
\begin{equation}
    \begin{split}
        \bbeta_t^{(i)} &= (1 - \gamma)  \G_{t-1}^{(i)} + \gamma \D_{t-1}^\top \M_t \x^{(i)} \\
        \G_t^{(i)} &= (1 - \gamma)  \G_{t-1}^{(i)} + \gamma \D_{t-1}^\top \M_t \D_t^{(i)},
    \end{split}
\end{equation}
where $\gamma$ is a weight
factor determined by the number of time the one sample $i$ has been previously observed at time
$t$. Precisely, given ${(\gamma_c)}_c$ a decreasing sequence of weights,
\begin{equation*}
    \gamma = \gamma_{c^{(i)}_t}\quad\text{where}\quad c^{(i)}_t =
     \left| \left\lbrace s \leq t, \x_s = \x^{(i)} \right\rbrace \right|.
\end{equation*}
All others estimators ${\{\G^{(j)}_t, \bbeta^{(j)}_t\}}_{j \neq i}$ are left unchanged from iteration
$t-1$. The set ${\{\G_t^{(i)}, \bbeta^{(i)}_t\}}_{1\leq i\leq n}$ is used to define
the \textit{averaged} estimators
\begin{equation}
    \label{eq:agg-estimates}
    \tag{b}
    \begin{split}
     \G_t &\triangleq  \G_t^{(i)} = \sum_{s \leq t, \x_s = \x^{(i)}} \gamma_{s,t}^{(i)} \D_{s-1}^\top \M_s \D_{s-1} \\
     \bbeta_t &\triangleq  \bbeta_t^{(i)} = \sum_{s \leq t, \x_s = \x^{(i)}} \gamma_{s,t}^{(i)} \D_{s-1}^\top \M_s \x^{(i)},
    \end{split}
\end{equation}
where
$\gamma_{s,t}^{(i)} = \gamma_{c^{(i)}_t} \prod_{s < t, \x_s = \x^{(i)}} (1 -
\gamma_{c^{(i)}_s})$. Using $ \bbeta_t$ and $ \G_t$ in~\eqref{eq:somf_alpha}, $\balpha_t$
minimizes the masked loss averaged over the previous iterations where sample
$i$ appeared:
\begin{equation}
    \label{eq:approx-regression}
     \min_{\balpha \in \RR^k} \sum_{\substack{s\leq t\\
     \x_s = \x^{(i)}}} \frac{\gamma_{s,t}^{(i)}}{2}
     \Vert \M_s(\x^{(i)} - \D_{s-1}^\top \balpha) \Vert_2^2
      + \lambda \Omega( \balpha ).
\end{equation}
The sequences ${(\G_t)}_t$ and ${(\bbeta_t)}_t$ are \textit{consistent}
estimations of ${(\G_t^\star)}_t$ and ${(\bbeta_t^\star)}_t$ --- consistency
arises from the fact that a single sample $\x^{(i)}$ is observed with different
masks along iterations. Solving~\eqref{eq:approx-regression} is made closer and
closer to solving~\eqref{eq:regression}, to ensure the
correctness of the algorithm (see Section~\ref{sec:analysis}). Yet, computing the
estimators~\eqref{eq:agg-estimates} is no more costly than
computing~\eqref{eq:estimates} and still permits to speed up a single iteration
close to $r$ times.
In the mini-batch setting, for every $i \in I_t$,
we use the estimators $\G_t^{(i)}$ and $\bbeta_t^{(i)}$
to compute $\balpha_t^{(i)}$. This method has a memory cost of $\mathcal{O}(n\,k^2)$. This is reasonable
compared to the dataset size\footnote{It is also possible to efficiently swap the estimators~$(\G_t^{(i)})_i$ on disk,
as they are only accessed for $i = i_t$ at iteration $t$.} if $p \gg k^2$.

\paragraph{Exact Gram computation}\label{par:gram}To
reduce the memory usage, another strategy is to use the \textit{true} Gram matrix
$\G_t$ and the estimator $ \bbeta_t$ from~\eqref{eq:agg-estimates}:
\begin{equation}
    \label{eq:gram-estimates}
    \tag{c}
    \begin{split}
         \G_t &\triangleq\G_t^\star = \D_{t-1}^\top \D_{t-1} \\
         \bbeta_t &\triangleq
        \sum_{s \leq t, \x_s = \x^{(i)}} \gamma_{s,t}^{(i)} \D_{s-1}^\top \M_s \x^{(i)}
    \end{split}
\end{equation}
As previously, the consistency of ${(\bbeta_t)}_t$ ensures that
\eqref{eq:expected-risk} is correctly solved despite the approximation in
${( \balpha_t)}_t$ computation. With the partial dictionary update step we propose, it
is possible to maintain $\G_t$ at a cost proportional to $q$. The time complexity of the coding step
is thus similarly reduced when replacing~\eqref{eq:agg-estimates} or~\eqref{eq:gram-estimates}
estimators in~\eqref{eq:regression}, but the latter option has a memory usage in
$\bigO(n\,k)$. Although estimators \eqref{eq:gram-estimates} are slightly less performant
in the first epochs, they are a good compromise between resource usage and convergence. We
summarize the characteristics of the three
estimators~\eqref{eq:estimates}--\eqref{eq:gram-estimates} in Table~\ref{table:estimators},
anticipating their empirical comparison in Section~\ref{sec:experiments}.

\begin{table}
    \centering
    \caption{Comparison of estimators used for code computation}
    \label{table:estimators}
    \begin{tabular}{cccccc}
        \toprule
        Est. & $\bbeta_t$ & $\G_t$ & Convergence & \specialcell{Extra\\mem. cost}
        & \specialcell{$1^{\textrm{st}}$ epoch\\perform.} \\
        \midrule
        \eqref{eq:estimates} & Masked & Masked & & &\checkmark \\
        \eqref{eq:agg-estimates} & Averaged & Averaged & \checkmark & $n\,k^2$ & \checkmark \\
        \eqref{eq:gram-estimates} & Averaged & Exact & \checkmark & $n\,k\phantom{^2}$ &  \\
        \bottomrule
    \end{tabular}
    \normalsize
\end{table}

\paragraph*{Surrogate computation}The computation of $\balpha_t$ using one of
the estimators above defines a surrogate $g_t(\D) \triangleq \frac{1}{2} \Vert
\x_t - \D \balpha_t \Vert_2^2 + \lambda \Omega(\balpha)$, which we use to update
the aggregated surrogate $\bar g_t \triangleq (1 - w_t) \bar g_{t-1} + w_t
g_t$, as in online matrix factorization. We follow \eqref{eq:omf-parameter-update}
(with weights ${(w_t)}_t$) to update the matrices $\bar \B_t$ and $\bar \C_t$,
which define $\bar g_t$ up to constant factors. The update of $\bar \B_t$ requires
a number of operations proportional to $p$.
Fortunately, we will see in the next paragraph that
it is possible to update $\P_t \bar
\B_t$ in the main thread with a number of operation proportional to $q$
and to complete the update of $\P_t^\perp \bar \B_t$ in parallel
with the dictionary update step.

\paragraph*{Weight sequences}\label{par:weights}Specific $(w_t)_t$ and $(\gamma_c)_c$ in Alg.~\ref{alg:somf}
are required. We provide then in Assumption~\ref{ass:gamma_decay} of the analysis: $w_t = \frac{1}{t^u}$ and $\gamma_c = \frac{1}{c^v}$,
 where $u \in (\frac{11}{12}, 1)$ and $v \in \big (\frac{3}{4}, 3u -2 \big)$ to ensure convergence. Weights have little impact on convergence speed in practice.

\subsubsection{Dictionary update}\label{sec:dict-update}

In the original online algorithm, the whole dictionnary $\D_{t-1}$ is updated
at iteration $t$. To reduce the time complexity of this step,
we add a ``freezing'' constraint to the minimization~\eqref{eq:bar_gt} of $\bar g_t$. Every
row $r$ of $\D$ that corresponds to an
\textit{unseen} row $r$ at iteration $r$ (such that $\M_t[r, r] = 0$) remains unchanged. This casts the
problem~\eqref{eq:bar_gt} into a lower dimensional space.
Formally, the
freezing operation comes out as a additional constraint
in~\eqref{eq:bar_gt}:
\begin{equation}
    \label{eq:dict-update-cons}
    \D_t = \argmin_{\substack{\D \in \mathcal{C}\\
    \P_t^\perp \D = \P_t^\perp \D_{t-1}}}
     \frac{1}{2} \trace (\D^\top \D \bar \C_t)
      - \trace (\D^\top \bar \B_t).
\end{equation}
The constraints are separable into two blocks of rows. Recalling the notations of~\eqref{eq:elastic_net}, for each atom $\d^{(j)}$, the rules $\Vert \d^{(j)} \Vert \leq 1$
and~$\P_t^\perp \d^{(j)} = \P_t^\perp \d^{(j)}_{t-1}$ can indeed be rewritten
\begin{equation}
    \left\lbrace
    \begin{array}{lll}
        \Vert \P_t \d^{(j)} \Vert &\leq& 1 - \Vert \d^{(j)}_{t-1} \Vert +
        \Vert \P_t \d^{(j)}_{t-1} \Vert \triangleq r_t^{(j)} \\
        \P_t^\perp \d^{(j)} &=& \P_t^\perp \d^{(j)}_{t-1}.
    \end{array}
    \right.
 \end{equation}
Solving~\eqref{eq:dict-update-cons} is therefore equivalent to
solving the following problem in $\RR^{q \times k}$, with $\B_t^r \triangleq \P_t \B_t$,
\begin{align}
    \label{eq:reduced_bcd}
    \D^r &\in \argmin_{\D^r \in \mathcal{C}^r}
     \frac{1}{2} \trace ({\D^r}^\top \D^r \bar \C_t) - \trace ({\D^r}^\top \bar \B_t^r)\\
     \text{where}\:\:\mathcal{C}^r &= \{\D^r {\in} \RR^{q \times k} /
     \forall j \in [0, k-1],
     \Vert \d^{r(j)} \Vert \leq r_t^{(j)}\} \notag.
\end{align}
The rows of $\D_t$ selected by $\P_t$ are then replaced with $\D^r$,
while the other rows of $\D_t$ are unchanged from iteration $t-1$.
Formally, $\P_t \D_t = \D^r$ and $\P_t^\perp \D_t = \P_t^\perp \D_{t-1}$.
 We solve~\eqref{eq:reduced_bcd} by
a projected block coordinate descent (BCD) similar to the one used in the original
algorithm, but performed in a subspace of size $q$. We compute each column $j$ of the gradient that we use
in the block coordinate descent loop with~$q\times k$ operations, as it writes $\D^r \bar \c_t^{(j)} - \bar \b_t^{r(j)} \in \RR^q$, where
$\bar \c_t^{(j)}$ and~$\bar \b_t^{r(j)}$ are the $j$-th columns of $\bar \C_t$ and $\bar \B_t^r$. Each reduced atom $\d^{r(j)}$ is projected onto the elastic-net ball of radius
$r_t^{(j)}$, at an average cost in $\bigO(q)$ following~\cite{mairal_online_2010}. This makes the complexity of a
single-column update proportional to $q$. Performing the
projection requires to keep in memory the values ${\{n_t^{(j)} \triangleq 1 - \Vert \d_t^{(j)} \Vert \}}_j$, which can be
updated online at a negligible cost.

\begin{algorithm}[t]
    \begin{algorithmic}
        \Input Dictionary $\D_{t-1}$, projector $\P_t$, statistics $\bar \C_t$,
        $\bar \B_t$, norms ${(n_{t-1}^{(j)})}_{0\leq j<k}$, Gram matrix $\G_t$ (optional).
        \State $\D_t \gets \D_{t-1}$, \qquad $\G_t \gets \G_t - \D_{t-1}^\top \P_t \D_{t-1}$.
        \For{$j \in \textrm{permutation}([1, k])$}
                \State $r_t^{(j)} \gets n_{t-1}^{(j)} + \Vert \P_t \d^{(j)}_{t-1} \Vert$.
                \State $\u \gets \P_t \d^{(j)}_{t-1} + \frac{1}{\bar \C_t[j,j]} (\P_t \bar \b_t^{(j)} - \P_t \D_t \bar \c_t^{(j)})$. \Comment in $\RR^q$
                \State $\P_t \d^{(j)}_t \gets \mathtt{enet\_projection}(\u, r_t^{(j)})$. \Comment in $\RR^q$
                \State $n_t^{(j)} \gets r_t^{(j)} - \Vert \P_t \d_t^{(j)} \Vert$.
        \EndFor
        \State $\G_{t+1} \gets \G_t + \D_{t}^\top \P_t \D_{t}$.
        \Output Dictionary $\D_t$, norms ${(n^{(j)}_t)}_{j}$,
        Gram matrix~$\G_{t+1}$. 
    \end{algorithmic}
    \caption{Partial dictionary update}\label{alg:somf-dictionary}
\end{algorithm}

We provide the reduced dictionary update
step in Alg.~\ref{alg:somf-dictionary}, where we use the function
$\mathtt{enet\_projection}(\u, r)$ that performs the orthogonal projection of
$\u \in \RR^q$ onto the elastic-net ball of radius $r$.  As in the original algorithm, we
perform a single pass over columns to solve
\eqref{eq:reduced_bcd}. Dictionary update is now performed with a number of
operations proportional to $q$, instead of $p$ in the original algorithm. Thanks to the
random nature of ${(\M_t)}_t$, updating $\D_{t-1}$ into $\D_t$ reduces~$\bar g_t$
enough to ensure convergence.

\paragraph*{Gram matrix computation}Performing partial updates of~$\D_t$ makes it possible to maintain the full
Gram matrix $\G_t~=~\G_t^\star$ with a cost in $\bigO(q\,k^2)$ per iteration, as mentioned in~\ref{par:gram}. It is indeed enough to compute the reduced Gram matrix~$\D^\top \P_t \D$ before and after the dictionary update:
\begin{equation}
    \G_{t+1} = \D_t^\top \D_t = \G_{t} - \D_{t-1}^\top \P_t \D_{t-1}^\top
    + \D_t^\top \P_t \D_t^\top.
\end{equation}
\paragraph*{Parallel surrogate computation} Performing block coordinate descent on $\bar g_t^r$ requires to access
$\bar \B_t^r = \P_t \bar \B_t$ only. Assuming we may use use more than two threads, this allows to parallelize
the dictionary update step with the update of~$\P_t^\perp \bar \B_t$. In the main thread, we compute
$\P_t \bar \B_t$ following
\begin{equation*}
    \tag{\ref{eq:somf_partial} -- Alg.~\ref{alg:somf}}
    \P_t \bar \B_t \gets (1 - w_t ) \bar \P_t \B_{t-1} + w_t \P_t \x_t  \balpha_t^\top.
\end{equation*}
which has a cost proportional to $q$. Then, we update in parallel the dictionary and the rows of $\bar \B_t$ that are not selected by $\M_t$:
\begin{equation*}
    \tag{\ref{eq:somf_minimization} -- Alg.~\ref{alg:somf}}
    \P_t^\perp \bar \B_t \gets (1 - w_t ) \P_t^\perp \bar \B_{t-1} + w_t \P_t^\perp \x_t  \balpha_t^\top.
\end{equation*}
This update requires $k (p - q) \eta$ operations (one matrix-matrix product)
for a mini-batch of size $\eta$.
In contrast, with appropriate implementation, the dictionary update step requires~$4\,k\,q^2$
 to~$6\,k\,q^2$ operations, among which $2\,k\,q^2$
come from slower matrix-vector products. Assuming $k \sim \eta$,
updating $\bar \B_t$ is faster than updating the dictionary up to $r \sim 10$,
and performing \eqref{eq:somf_minimization} on a second thread is seamless
 in term of wall-clock time. More threads may be used for larger reduction or
batch size.

%

\subsubsection{Subsampling and time complexity}\label{par:tradeoff}

Subsampling may be used in only \textit{some} of the steps of
Alg.~\ref{alg:somf}, with the other steps following Alg.~\ref{alg:omf}. Whether
to use subsampling or not in each step depends on the trade-off between the computational speed-up
it brings and the approximations it makes. It is useful to understand how complexity of \omf evolves
with~$p$. We write~$s$ the average number of non-zero coefficients in
${(\balpha_t)}_t$ ($s = k$ when $\Omega = \Vert \cdot \Vert_2^2$). \omf complexity has three terms:
\begin{enumerate}[(i)]
    \item $\bigO(p\,k^2)$: computation of the Gram matrix $\G_t$, update of the dictionary $\D_t$ with block coordinate descent,
    \item $\bigO(p\,k\,\eta)$: computation of $\bbeta_t = \D_{t-1}^\top \x_t$ and of $\bar \B_t$ using~$\x_t \balpha_t^\top$,
    \item $\bigO(k\,s^2\,\eta)$: computation of $\balpha_t$ using $\G_t$ and $\bbeta_t$, using matrix inversion or elastic-net regression.
\end{enumerate}
Using subsampling turns $p$ into $q = \frac{p}{r}$ in the expressions above. It
improves single iteration time when the cost of regression
$\bigO(k\,s^2\,\eta)$ is dominated by another term. This happens whenever
$\frac{p}{r} > s^2$, where $r$ is the reduction factor used in the algorithm.
Subsampling can bring performance improvement up to $r \sim \frac{p}{s^2}$. It
can be introduced in either computations from (i) or (ii), or both. When using
small batch size, i.e., when $\eta < k$, computations from (i) dominates
complexity, and subsampling should be first introduced in dictionary update
(i), and for code computation (ii) beyond a certain reduction ratio. On the
other hand, with large batch size $\eta > k$, subsampling should be first
introduced in code computation, then in the dictionary update step. The reasoning above ignore
potentially large constants. The best trade-offs in using subsampling must be empirically
determined, which we do in Section~\ref{sec:experiments}.

\subsection{Stochastic approximate
majorization-minimization}\label{stochastic-approximate-majorization-minimization}

The \somf algorithm can be understood within the
stochastic majorization-minimization framework. The modifications that we propose are
indeed perturbations to the first and third steps of the \smm presented in Algorithm~\ref{alg:smm}:
\begin{itemize}
    \item The code is computed approximately: the surrogate
    is only an \textit{approximate} majorizing surrogate of $f_t$ near~$\D_{t-1}$.
    \item The surrogate objective is only \textit{reduced} and not minimized, due
    to the added constraint and the fact that we perform only one pass of block coordinate descent.
\end{itemize}

We propose a new \textit{stochastic approximate
majorization-minimization} (\samm) framework handling these perturbations:
\begin{itemize}
	\item A majorization step~(\ref{eq:majorization_step_smm} -- Alg.~\ref{alg:smm}),
	 computes an \textit{approximate surrogate} of $f_t$ near $\theta_{t-1}$: $
	        g_t \approx g_t^\star$, where $g_t$ is a true upper-bounding surrogate of $\bar f_t$.
	\item A minimization step~(\ref{eq:minimization_step_smm} --
Alg.~\ref{alg:smm}), finds $\theta_t$
	by reducing \textit{enough} the objective $\bar g_t$: $\theta_t \approx \theta_t^\star \triangleq \argmin_{\theta \in \Theta}
			 \bar g_t(\theta)$, which implies $\bar g_t(\theta_t) \gtrsim \bar g_t(\theta_t^\star)$.
\end{itemize}
The \samm framework is general, in the sense that approximations are not
specified. The next section provides a theoretical
analysis of the approximation of \samm and establishes how \somf is an
instance of \samm.
It concludes by establishing Proposition~\ref{prop:somf}, which provides convergence
guarantees for \somf, under the same assumptions made for \omf
in~\cite{mairal_online_2010}.

\section{Convergence analysis}\label{sec:analysis}

We establish the convergence of \somf under reasonable assumptions.
For the sake of clarity, we first state our principal result (Proposition~\ref{prop:somf}), that guarantees \somf convergence. It is a corollary of a more general result on \samm algorithms. To present this broader result, we recall the theoretical
guarantees of the stochastic majorization-minimization algorithm~\cite{mairal_stochastic_2013} (Proposition~\ref{prop:param}); then, we show how the algorithm can withstand pertubations (Proposition~\ref{prop:approximate-surrogate}). Proofs are reported in Appendix~\ref{app:demos}. \samm convergence is proven before establishing \somf convergence as a corollary of this broader result.

\subsection[Convergence of SOMF]{Convergence of \somf}

Similar to~\cite{mairal_online_2010, mardani_subspace_2015}, we show that the sequence of iterates $(\D_t)_{t}$ asymptotically reaches a critical point of the empirical risk~\eqref{eq:empirical-risk}. We introduce the same hypothesis on the code covariance estimation $\bar \C_t$ as in \cite{mairal_online_2010} and a similar one on $\G_t$ --- they ensure strong convexity of the surrogate and boundedness of~${(\balpha_t)}_t$. They do not cause any loss of generality as they are met in practice after a few iterations, if $r$ is chosen reasonably low, so that $q > k$. The following hypothesis can also be guaranteed by adding small $\ell_2$ regularizations to~$\bar f$.
\begin{assumption}\label{ass:D_cond}
    There exists $\rho > 0$ such that for all $t > 0$, $\bar \C_t, \G_t \succ \rho \I$. 
\end{assumption}
We further assume, that the weights ${(w_t)}_t$ and ${(\gamma_c)}_c$ decay at specific rates. We specify simple weight sequences, but the proofs can be adapted for more complex ones.

\begin{assumption}\label{ass:gamma_decay}
    There exists $u \in (\frac{11}{12}, 1)$ and $v \in \big (\frac{3}{4}, 3 u - 2)$ such that, for all $t > 0$, $c > 0$, $w_t = t^{-u}$, $\gamma_c \triangleq c^{-v}$.
\end{assumption}

The following convergence result then applies to any sequence ${(\D_t)}_t$ produced by \somf, using estimators~\eqref{eq:agg-estimates} or \eqref{eq:gram-estimates}. $\bar f$ is the empirical risk defined in~\eqref{eq:empirical-risk}.

\begin{proposition}[\somf convergence]\label{prop:somf}

Under assumptions~\ref{ass:D_cond} and~\ref{ass:gamma_decay},
$\bar f(\D_t)$ converges with probability one and every limit point $\D_\infty$ of ${(\D_t)}_t$ is a stationary point of $\bar f$: for all $\D \in \mathcal{C}$
\begin{equation}
    \nabla \bar f(\D_\infty, \D - \D_\infty) \geq 0
\end{equation}
\end{proposition}

This result applies for any positive subsampling ratio $r$, which may be set arbitrarily high.
However, selecting a reasonable ratio remains important for performance.

Proposition~\ref{prop:somf} is a corollary of a stronger result on \samm
algorithms. As it provides insights on the convergence mechanisms, we
formalize this result in the following.

\subsection[{Basic assumptions and results on SMM convergence}]
{Basic assumptions and results on \smm convergence}

We first recall the main results on stochastic majorization-minimization algorithms,
established in \cite{mairal_stochastic_2013}, under assumptions that we
slightly tighten for our purpose. In our setting, we consider the empirical risk minimization problem
\begin{equation}
    \label{eq:empirical-risk-general}
    \min_{\theta \in \Theta} \Big(\bar f(\theta) \triangleq \frac{1}{n} \sum_{i=1}^n f(\theta, \x^{(i)}) \Big),
\end{equation}
where $f : \RR^K \times \mathcal{X} \to \RR$ is a loss function and

\begin{assumption}\label{ass:compact}
    $\Theta \subset \RR^K$ and the support $\mathcal{X}$ of the data are compact.
\end{assumption}

This is a special case of~\eqref{eq:expected-risk} where the samples ${(\x_t)}_t$ are drawn uniformly from the set $\{ \x^{(i)} \}_i$.
The loss functions $f_t \triangleq f(\cdot, \x_t)$ defined on
$\RR^K$ can be non-convex. We instead assume that they meet reasonable
regularity conditions:

\begin{assumption}\label{ass:f_lip}

    ${(f_t)}_t$ is uniformly $R$-Lipschitz continuous on $\RR^K$ and uniformly bounded on $\Theta$.

\end{assumption}

\begin{assumption}\label{ass:directional_derivative}

    The directional derivatives~\cite{borwein_convex_2010} $\nabla f_t(\theta, \theta' - \theta)$ and
    $\nabla \bar f(\theta, \theta' - \theta)$ exist for all $\theta$ and $\theta'$
    in $\RR^K$.

\end{assumption}

Assumption \ref{ass:directional_derivative} allows to characterize the \textit{stationary points}
of problem~\eqref{eq:empirical-risk-general}, namely $\theta \in \Theta$
such that $\nabla \bar f(\theta, \theta' - \theta) \geq 0$ for all $\theta' \in \Theta$ --- intuitively a point is stationary when there is no local direction in which the objective can be improved.

Let us now
recall the definition of first-order surrogate
functions used in the \smm algorithm. ${(g_t)}_t$ are selected in the set
$\mathcal{S}_{\rho, L}(f_t, \theta_{t-1})$, hereby introduced.

\begin{definition}[First-order surrogate function]\label{def:surrogate}
    Given a function $f : \RR^K \to \RR$, $\theta \in \Theta$ and $\rho, L > 0$,
    we define $\mathcal{S}_{\rho, L}(f, \theta)$ as the set of functions $g : \RR^K \rightarrow \RR$ such
    that
    \begin{itemize}
        \item $g$ is majorizing $f$ on $\Theta$ and $g$ is $\rho$-strongly convex,
        \item $g$ and $f$ are tight at $\theta$ --- i.e., $g(\theta) = f(\theta)$,
         $g - f$ is differentiable,
         $\nabla (g - f)$ is $L$-Lipschitz,
         $\nabla (g - f)(\theta) = 0$.
    \end{itemize}
\end{definition}

In \omf, $g_t$ defined in~\eqref{eq:surrogate} is a \textit{variational} surrogate\footnote{In this case as in \somf, $g_t$ is not $\rho$-strongly convex but $\bar g_t$ is, thanks to assumption~\ref{ass:D_cond}. This is sufficient in the proofs of convergence.} of~$f_t$. We refer the reader to~\cite{mairal_optimization_2013} for further examples of first-order surrogates. We also ensure that $\bar g_t$ should be \textit{parametrized} and thus representable in memory. The following assumption is met in \omf, as $\bar g_t$ is parametrized by the matrices~$\bar \C_t$ and~$\bar \B_t$.

\begin{assumption}[Parametrized surrogates.]\label{ass:param}

    The surrogates $(\bar g_t)_t$ are parametrized by
    vectors in a compact set $\mathcal{K} \subset \RR^P$. Namely, for all
    $t > 0$, there exists $\bkappa_t \in \mathcal{K}$ such that $\bar g_t$ is unequivocally defined
    as $g_t \triangleq \bar g_{\bkappa_t}$.

\end{assumption}

Finally, we ensure that the weights ${(w_t)}_t$ used in Alg.~\ref{alg:smm}
decrease at a certain rate.

\begin{assumption}
    \label{ass:weights_decay}
    There exists $u \in (\frac{3}{4}, 1)$ such that $w_t = t^{-u}$.
\end{assumption}

When $(\theta_t)_t$ is the sequence yielded by Alg.~\ref{alg:smm}, the
following result (Proposition 3.4 in \cite{mairal_stochastic_2013}) establishes
the convergence of~$(\bar f{(\theta_t))}_t$ and states that $\theta_t$ is
asymptotically a stationary point of the finite sum problem~\eqref{eq:empirical-risk-general}, as a special case of the \textit{expected} risk minimization problem~\eqref{eq:expected-risk}.

\begin{proposition}[Convergence of \smm, from~\cite{mairal_stochastic_2013}]
    \label{prop:param}
    Under assumptions~\ref{ass:compact} --~\ref{ass:weights_decay},
    ${(\bar f(\theta_t))}_{t \geq
    1}$ converges with probability one. Every limit point $\theta_{\infty}$ of ${(\theta_t)}_t$
    is a stationary point of the risk $\bar f$ defined
in~\eqref{eq:empirical-risk-general}. That is,
    \begin{equation}
        \label{eq:param_accumulation}
        \forall \theta \in \Theta,  \quad \nabla \bar f(\theta_\infty, \theta -\theta_\infty) \geq 0.
    \end{equation}
\end{proposition}

The correctness of the online matrix factorization algorithm can be deduced from
this proposition.

\subsection{Convergence of \samm}%
\label{approximate-surrogate-in-the-majorization-step}

We now introduce assumptions on the approximations
made in \samm, before extending the result of Proposition~\ref{prop:param}.
We make hypotheses on both the surrogate computation (\textit{majorization}) step
and the iterate update (\textit{minimization}) step. The principles of \samm are illustrated in Figure~\ref{fig:samm}, which
provides a geometric interpretation of the approximations introduced in the
following assumptions~\ref{ass:epsilon_decay} and~\ref{ass:geometric}.

\subsubsection{Approximate surrogate computation}

The \smm algorithm selects a surrogate for
$f_t$ at point $\theta_{t-1}$ within the set $\mathcal{S}_{\rho, L}(f_t,
\theta_{t-1})$. Surrogates within this set are \textit{tight} at $\theta_{t-1}$
and greater than $f_t$ everywhere. In \samm, we allow the use of surrogates that are only
\textit{approximately majorizing} $f_t$ and \textit{approximately tight} at $\theta_{t-1}$.
This is indeed what \somf does when using estimators in the code computation step.
For that purpose, we introduce the set
$\mathcal{T}_{\rho, L}(f, \theta, \epsilon)$, that contains all functions
$\epsilon$\=/close of a surrogate in $\mathcal{S}_{\rho, L}(f, \theta)$ for the
$\ell_\infty$\=/norm:

\begin{definition}[Approximate first-order surrogate function]\label{def:approx-surrogate}
    Given a function $f : \RR^K \to \RR$, $\theta \in \Theta$ and $\epsilon > 0$, $\mathcal{T}_{\rho,
    L}(f, \theta, \epsilon)$ is the set of $\rho$-strongly convex functions $g : \RR^K \rightarrow \RR$ such~   that
    \begin{itemize}
        \item $g$ is $\epsilon$-majorizing $f$ on $\Theta$:
            $\forall\:\kappa \in \Theta,\, g(\kappa) - f(\kappa) \geq - \epsilon$,
        \item $g$ and $f$ are $\epsilon$-tight at $\theta$ ---
        i.e., $g(\theta) - f(\theta) \leq \epsilon$,
         $g - f$ is differentiable,
         $\nabla (g - f)$ is $L$-lipschitz.
    \end{itemize}
\end{definition}

We assume that \samm selects an approximative surrogate in $\mathcal{T}_{\rho,
L}(f_t, \theta_{t-1}, \epsilon_t)$ at each iteration, where ${(\epsilon_t)}_t$ is a deterministic or random non-negative sequence
 that vanishes at a sufficient rate.

\begin{assumption}\label{ass:epsilon_decay}
    For all $t > 0$, there exists $\epsilon_t > 0$ such that $g_t \in \mathcal{T}_{\rho, L}(f_t, \theta_{t-1}, \epsilon_t)$. There exists a constant $\eta >0$ such that
    $\EE [\epsilon_t] \in
    \mathcal{O}(t^{2(u-1) - \eta})$ and $\epsilon_t
    \to_\infty 0$ almost surely.
\end{assumption}
As illustrated on Figure~\ref{fig:samm}, given the \omf surrogate $g_t^\star \in \mathcal{S}_{\rho, L}(f_t, \theta_{t-1})$ defined in~\eqref{eq:surrogate}, any function $g_t$ such that $\Vert g_t -~g_t^\star \Vert_\infty < \epsilon$ is in
$\mathcal{T}_{\rho, L}(f_t, \theta_{t-1}, \epsilon)$ --- \textit{e.g.},
where $g_t$ uses an approximate $\balpha_t$
in~\eqref{eq:surrogate}. This assumption can also be met in matrix factorization settings with difficult code regularizations, that \textit{require} to make code approximations.

\subsubsection{Approximate surrogate minimization}

We do not
require~$\theta_t$ to be the minimizer of $\bar g_t$ any longer, but ensure that
the surrogate objective function $\bar g_t$ decreases ``fast enough''.
Namely, $\theta_t$ obtained from partial minimization should be closer to a minimizer of $\bar g_t$
than $\theta_{t-1}$. We write ${(\mathcal{F}_{t})}_t$ and ${(\mathcal{F}_{t-\frac{1}{2}})}_t$ the
filtrations induced by the past of the algorithm, respectively up to the end of iteration $t$ and up to the beginning of the minimization step
in iteration $t$. Then, we assume

\begin{assumption}\label{ass:geometric}
    For all $t > 0$, $\bar g_t(\theta_t) < \bar g_t(\theta_{t-1})$. There exists~$\mu > 0$ such that, for all $t > 0$,
    where $\theta_t^\star = \argmin_{\theta \in \Theta} \bar g_t(\theta)$,
    \begin{equation}
        \label{eq:stochastic_geometric_decrease}
        \EE [\bar g_t(\theta_t) - \bar g_t(\theta_t^\star) | \mathcal{F}_{t-\frac{1}{2}}] \leq (1 -
        \mu) (\bar g_t(\theta_{t-1}) - \bar g_t(\theta_{t}^\star)).
    \end{equation}
\end{assumption}
\setcounter{stoass}{\value{assumption}}

\begin{figure}
    \centering
    \includegraphics{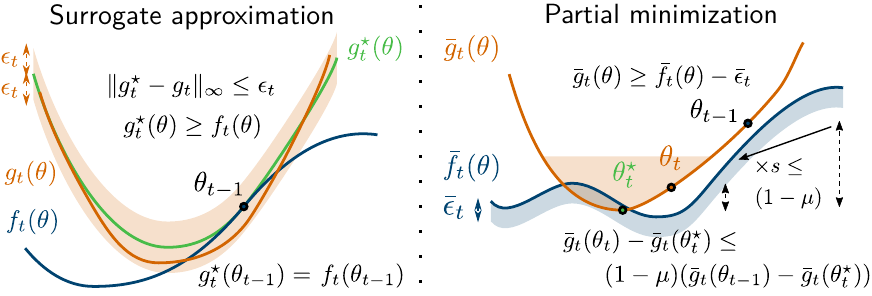}
    \caption{Both steps of \samm make well-behaved approximations. The operations
    that are performed in exact \smm are in green and superscripted by $^\star$, while the actual computed values are
    in orange. Light bands recall the bounds on approximations assumed in~\ref{ass:epsilon_decay} and \ref{ass:geometric}.}
    \label{fig:samm}
\end{figure}

Assumption~\ref{ass:geometric} is met by choosing an appropriate method for the inner $\bar g_t$
minimization step --- a large variety of gradient-descent
algorithms indeed have convergence rates of the form~\eqref{eq:stochastic_geometric_decrease}.  In \somf,
the block coordinate descent with frozen coordinates indeed meet this property, relying on results from~\cite{wright_coordinate_2015}.
 When both assumptions are met, \samm enjoys the same convergence guarantees as \smm.
\subsubsection{Asymptotic convergence guarantee}
The following proposition guarantees that the stationary point condition
of Proposition~\ref{prop:param} holds for the \samm algorithm, despite the use of approximate surrogates and approximate minimization.
\begin{proposition}[Convergence of \samm]
    \label{prop:approximate-surrogate}
    Under assumptions~\ref{ass:compact} --~\ref{ass:geometric}, the conclusion of Proposition~\ref{prop:param} holds for~\samm.
\end{proposition}

%
%
Assumption~\ref{ass:epsilon_decay} is essential to bound the errors introduced
by the sequence $(\epsilon_t)_t$ in the proof of
Proposition~\ref{prop:approximate-surrogate}, while~\ref{ass:geometric} is the
key element to show that the sequence of iterates $(\theta_t)_t$ is stable enough to ensure convergence. The result holds
for any subsampling ratio $r$, provided that~\ref{ass:D_cond} remains true.

\subsubsection{Proving \somf convergence}Assumptions~\ref{ass:D_cond} and~\ref{ass:gamma_decay} readily implies~\ref{ass:compact}--\ref{ass:weights_decay}.
With Proposition~\ref{prop:approximate-surrogate}
at hand, proving Proposition~\ref{prop:somf} reduces to ensure that the surrogate
sequence of \somf meets~\ref{ass:epsilon_decay} while its iterate sequence meets~\ref{ass:geometric}.

\section{Experiments}\label{sec:experiments}

The \somf algorithm is designed for datasets with large number of samples
$n$ and large dimensionality $p$. Indeed, as detailed in Section \ref{sec:somf},
 subsampling
removes the computational
bottlenecks that arise from high dimensionality. Proposition~\ref{prop:somf} establishes
that the subsampling used in \somf is safe, as it enjoys the same guarantees as \omf.
However, as with \omf, no convergence rate is provided. We therefore perform a strong
 empirical validation of subsampling.

We tackle two different problems, in
functional Magnetic Resonance Imaging (fMRI) and hyperspectral imaging.
Both involve the factorization of very large matrices $\X$ with sparse factors.
As the data we consider are huge, subsampling reduces the time of a single
iteration by a factor close to $\frac{p}{q}$. Yet it is also much redundant: \somf makes little approximations and
accessing only a fraction of the features per iteration should not hinder much
the refinement of the dictionary. Hence high speed-ups are
expected --- and indeed obtained. All experiments can be reproduced using open-source code.

\subsection{Problems and datasets}

\subsubsection{Functional MRI}

Matrix factorization has long been used on functional Magnetic Resonance
Imaging \cite{mckeown_analysis_1998}. Data are temporal series of 3D images of
brain activity and are decomposed into spatial modes capturing regions that activate
synchronously. They form a matrix $\X$ where columns are the 3D images, and rows
corresponds to voxels.
Interesting dictionaries for neuroimaging capture
spatially-localized components, with a few brain regions. This can be obtained
by enforcing sparsity on the dictionary: we use an $\ell_2$ penalty and the
elastic-net constraint. \somf streams subsampled 3D brain records to learn the
sparse dictionary $\D$. Data can be huge: we use the whole HCP dataset
\cite{van_essen_wu-minn_2013}, with $n=2.4\cdot 10^6$ (2000 records, 1\,200
time points) and $p=2\cdot10^5$, totaling 2\,TB of dense data. For comparison,
we also use a smaller public dataset (ADHD200 \cite{milham_adhd-200_2012}) with
40 records, $n=7000$~samples and $p=6\cdot10^4$ voxels. Historically, brain
decomposition have been obtained by minimizing the classical dictionary learning
objective on \textit{transposed} data \cite{varoquaux_cohort-level_2013}: the
code $\A$ holds sparse spatial maps and voxel time-series are streamed. This is
not a natural streaming order for fMRI data as $\X$ is stored columnwise on
disk, which makes the sparse dictionary formulation more appealing.
Importantly, we seek a \textit{low-rank} factorization, to keep the
decomposition interpretable --- $k \sim 100 \ll p$.

\subsubsection{Hyperspectral imaging}Hyperspectral cameras acquire images with
many channels that correspond to different spectral bands. They are used
heavily in remote sensing (satellite imaging), and material study (microscopic
imaging). They yield digital images with around $1$ million pixels, each
associated with hundreds of spectral channels. Sparse matrix factorization has
been widely used on these data for image classification
\cite{chen_hyperspectral_2011, soltani-farani_spatial-aware_2015} and denoising
\cite{maggioni_nonlocal_2013, peng_decomposable_2014}. All
methods rely on the extraction of full-band patches representing a local image neighborhood with
all channels included. These patches are very high dimensional, due to the
number of spectral bands. From one image of the AVIRIS project
\cite{vane_first_1987}, we extract $n = 2\cdot10^6$ patches of size $16\times
16$ with $224$ channels, hence $p = 6\cdot 10^4$. A dense dictionary
is learned from these patches. It should allow a sparse representation of
samples: we either use the classical dictionary learning setting
($\ell_1$/elastic-net penalty), or further add positive constraints to the dictionary and codes: both methods may be used and deserved to be benchmarked.
We seek a dictionary of reasonable size: we use $k \sim 256 \ll p$.

\subsection{Experimental design}

To validate the introduction of subsampling and the usefulness of \somf, we perform two major experiments.
\begin{itemize}
    \item We measure the performance of \somf when increasing the reduction factor,
    and show benefits of stochastic dimension reduction on all datasets.
    \item We assess the importance
    of subsampling in each of the steps of \somf. We compare the different approaches proposed for code
    computation.
\end{itemize}

\begin{table}
    \caption{Summary of experimental settings}\label{table:experiments}%
    \centering
    \begin{tabular}{lccc}
        \toprule
        Field & \multicolumn{2}{c}{Functional MRI} & Hyperspectral imaging \\
        \cmidrule(r){2-3}
        Dataset & ADHD & HCP & Patches from AVIRIS \\
        \midrule
        Factors & \multicolumn{2}{c}{$\D$ sparse, $\A$ dense} & $\D$ dense, $\A$ sparse \\
        \# samples $n$ & $7\cdot10^3$ & $2\cdot10^6$ & $ 2\cdot10^6$ \\
        \# features $p$ & $6\cdot10^4$ & $2\cdot10^5$ & $6\cdot10^4$ \\
        $\X$ size & 2 GB & 2 TB & 103 GB \\
        Use case ex. & \multicolumn{2}{c}{Extracting predictive feature} & Recognition / denoising \\
        \bottomrule
    \end{tabular}
\end{table}

\subsubsection*{Validation}We compute the objective function
$\eqref{eq:empirical-risk}$ over a test set to rule out any overfitting effect ---
a dictionary should be a good representation of unseen samples. This criterion
is always plotted against wall-clock time, as we are interested in the performance of
\somf for practitioners.

\begin{figure*}[ht]
    \centering
    \includegraphics[width=\textwidth]{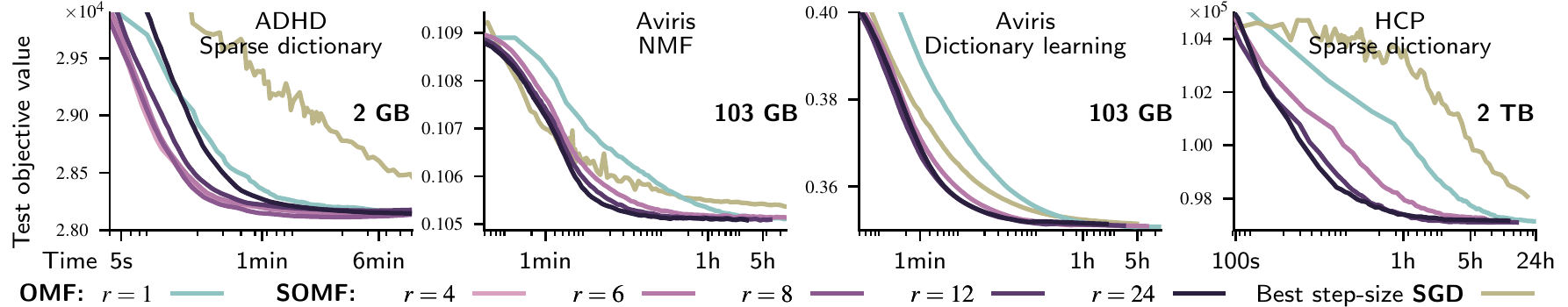}%
    \caption{Subsampling provides significant speed-ups on all fMRI and
    hyperspectral datasets. A reduction factor of $12$ is a good
    overall choice.
    With larger data, larger reduction factors can be used for better
performance --- convergence is reached $13 \times$ faster than
state-of-the-art methods on the 2TB HCP dataset.}\label{fig:bench}
\end{figure*}

\subsubsection*{Tools}To perform a valid benchmark, we implement \omf and \somf
using Cython \cite{behnel_cython:_2011}
We use coordinate descent~\cite{friedman_pathwise_2007} to
solve Lasso problems with optional positivity constraints. Code computation is parallelized to handle mini-batches. Experiments use
\textit{scikit-learn}~\cite{pedregosa_scikit-learn:_2011} for numerics,
and \textit{nilearn}~\cite{abraham_machine_2014} for handling fMRI data. We have released
the code
in an open-source Python package\footnote{\url{https://github.com/arthurmensch/modl}}. Experiments were run on 3 cores of an
Intel Xeon 2.6GHz, in which case computing $\P_t^\perp \bar \B_t$ is faster than updating $\P_t \D_t$.

\subsubsection*{Parameter setting}Setting the number of components $k$ and the
amount of regularization $\lambda$ is a hard problem in the absence of ground
truth. Those are typically set by cross-validation when matrix factorization is
part of a supervised pipeline. For fMRI, we set $k = 70$ to obtain
interpretable networks, and set~$\lambda$ so that the decomposition
approximately covers the whole brain (i.e., every map is $\frac{k}{70})$
sparse). For hyperspectral images, we set $k = 256$ and select $\lambda$ to
obtain a dictionary on which codes are around $3\%$ sparse. We cycle randomly through
the data (fMRI records, image patches) until convergence, using mini-batches of
size $\eta = 200$ for HCP and AVIRIS, and $\eta = 50$ for ADHD (small number of
samples). Hyperspectral patches are normalized in the dictionary learning
setting,
but not in the non-negative setting --- the classical pre-conditioning for each case. We use $u = 0.917$ and $v = 0.751$ for weight sequences.
%


\subsection{Reduction brings speed-up at all data scales}

We  benchmark \somf for various reduction factors against the original online
matrix factorization algorithm \omf \cite{mairal_online_2010}, on the three
presented datasets. We stream data in the same order for all reduction factors.
Using variant \eqref{eq:gram-estimates} (true Gram matrix, averaged $\bbeta_t$)
performs slightly better on fMRI datasets, whereas \eqref{eq:agg-estimates}
(averaged Gram matrix and $\bbeta_t$) is slightly faster for hyperspectral
decomposition. For comparison purpose, we display results using estimators~\eqref{eq:agg-estimates} only.

Figure~\ref{fig:bench} plots the test objective against CPU time. First, we observe
that all algorithms find dictionaries with very close objective function values
for all reduction factors, on each dataset. This is not a trivial
observation as the matrix factorization problem \eqref{eq:empirical-risk} is
not convex and different runs of \omf and \somf may converge towards minima
with different values. Second, and most importantly, \somf provides significant
improvements in convergence speed for three different sizes of data and three
different factorization settings. Both observations
confirm the relevance of the subsampling approach.

Quantitatively, we
summarize the speed-ups obtained in Table~\ref{table:speed-ups}.
On fMRI data, on both large and medium datasets, \somf provides more than an
order of magnitude speed-up. Practitioners working on datasets akin to HCP can
decompose their data in 20 minutes instead of $4\,\textrm{h}$ previously, while
working on a single machine. We obtain the highest speed-ups for the largest
dataset --- accounting for the extra redundancy that usually appears when
dataset size increase. Up to $r \sim 8$, speed-up is of the
order of~$r$ --- subsampling induces little noise in the iterate sequence,
compared to \omf. Hyperspectral decomposition is performed near
$7\times$ faster than with \omf in the classical dictionary learning setting, and $3 \times$ in the non-negative setting,
which further demonstrates the versatility of \somf. Qualitatively, given a
certain time budget, Figure~\ref{fig:qualitative} compares the results of \omf
and the results of \somf with a subsampling ratio $r=24$, in the non-negative setting. Our algorithm yields
a valid smooth bank of filters much faster.
The same comparison has been made for fMRI in~\cite{mensch_dictionary_2016}.

\subsubsection*{Comparison with stochastic gradient descent}It is possible to
solve~\eqref{eq:empirical-risk} using the projected stochastic gradient
(\textsc{sgd}) algorithm~\cite{duchi_efficient_2008}. On all tested settings,
for high precision convergence,
\textsc{sgd} (with the best step-size among a grid) is slower than \omf and
even slower than \somf. In the dictionary learning setting, \textsc{sgd}
is somewhat faster than \omf but slower than \somf in the first
epochs. Compared to \somf and \omf, \textsc{sgd}
further requires to select the step-size by grid search.

%

\begin{figure}
    \centering
    \includegraphics{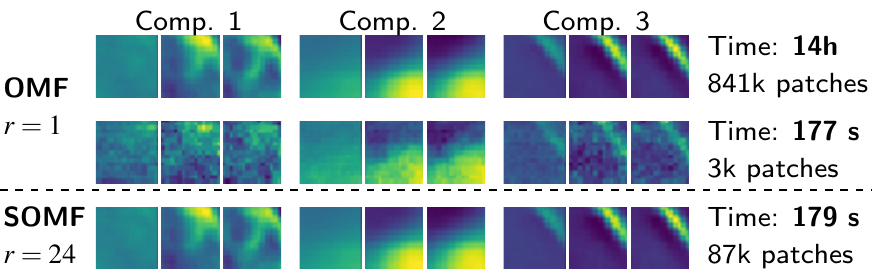}
    \caption{Given a 3 minute time budget,
    the atoms learned by \somf are more focal and less noisy that those learned by \omf. They are closer to
    the dictionary of first line, for which convergence has been reached.}\label{fig:qualitative}
\end{figure}

\subsubsection*{Limitations}Table \ref{table:speed-ups}
reports convergence time within
$1\%$, which is enough for application in practice. \somf is less beneficial
when setting very high precision: for convergence within $0.01 \%$, speed-up
for HCP is $3.4$. This is expected as \somf trades speed for approximation. For
high precision convergence, the reduction ratio can be reduced after a few epochs.
As expected, there exists an \textit{optimal} reduction ratio, depending on the problem and precision,
 beyond which performance reduces: $r =
12$ yields better results than $r=24$ on AVIRIS (dictionary learning) and ADHD, for $1\%$ precision.

\begin{table}
    \caption{Time to reach convergence ($<1\%$ test objective)}%
    \label{table:speed-ups}
    \centering
	\setlength{\tabcolsep}{2.5pt}
    \begin{tabular}{clrlrlrlr}
        \toprule
        Dataset & \multicolumn{2}{c}{ADHD} & \multicolumn{2}{c}{AVIRIS (NMF)} &
		\multicolumn{2}{c}{AVIRIS (DL)} &
		\multicolumn{2}{c}{HCP} \\

        Algorithm & \omf & \somf & \omf & \somf & \omf & \somf & \omf & \somf \\
        \cmidrule(r){1-1}
        \cmidrule(lr){2-3}
        \cmidrule(lr){4-5}
        \cmidrule(lr){6-7}
        \cmidrule(l){8-9}
        Conv. time &
        $6\,\textrm{min}$ & $\mathbf{28\,\textbf{s}}$ &
          $2\,\textrm{h}\,30$ & $\mathbf{43\,\textbf{min}}$ &
		  $1\,\textrm{h}\,16$& $\mathbf{11\,\textbf{min}}$&
           $3\,\textrm{h}\,50$ & $\mathbf{17\,\textbf{min}}$ \\
        Speed-up & \multicolumn{2}{c}{11.8} &
		 \multicolumn{2}{c}{3.36}&
		 \multicolumn{2}{c}{6.80} &
         \multicolumn{2}{c}{13.31} \\
        \bottomrule
    \end{tabular}
\end{table}

Our first experiment establishes the power of stochastic subsampling as a
whole. In the following two experiments, we refine our analysis to show
that subsampling is indeed useful in the three steps of online matrix
factorization.

\begin{figure}[t]
    \centering
    \includegraphics{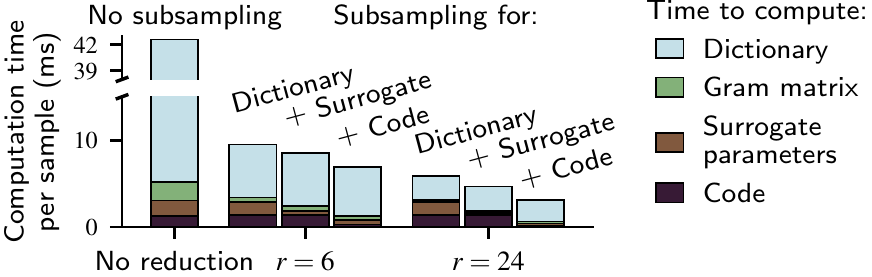}%
    \caption[Profiling \omf and \somf on HCP.]%
    {Profiling \omf and \somf on HCP. Partial dictionary update removes the major bottleneck of online matrix factorization for small reductions. For higher reduction,
    parameter update and code computation must be subsampled to further reduce the iteration time.}\label{fig:profiling}
\end{figure}

\subsection{For each step of \somf, subsampling removes a bottleneck}

\label{sec:profiling}

In Section~\ref{sec:stochastic-subsampling}, we have provided theoretical
guidelines on when to introduce subsampling in each of the three steps of an
iteration of \somf. This analysis predicts that, for $\eta \sim k$, we should
first use partial dictionary update, before using approximate code computation
and asynchronous parameter aggregation. We verify this by measuring the time
spent by \somf on each of the updates for various reduction factors, on the HCP
dataset. Results are presented in Figure~\ref{fig:profiling}. We observe that
block coordinate descent is indeed the bottleneck in \omf. Introducing partial
dictionary update removes this bottleneck, and as the reduction factor
increases, code computation and surrogate aggregation becomes the major
bottlenecks. Introducing subsampling as described in \somf overcomes these
bottlenecks, which rationalizes all steps of \somf from a computational point of view.

\subsection{Code subsampling becomes useful for high reduction}

\label{sec:empirical_high_reduction}

\begin{figure*}[t]
    \centering
    \includegraphics[width=\textwidth]{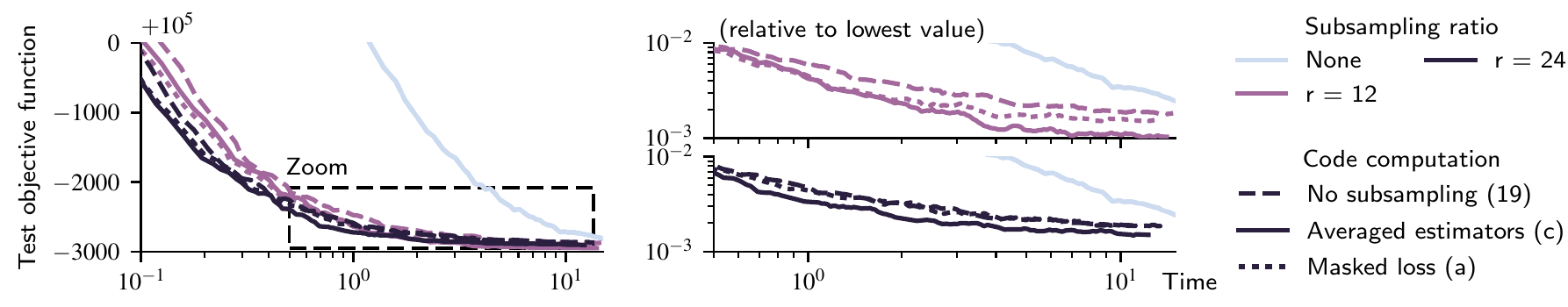}%
    \vspace{-.5em}
    \caption{Approximating code computation with the proposed subsampling method further accelerates the convergence of \somf. Refining code computation
    using past iterations (averaged estimates) performs better than simply performing a subsampled linear regression as in~\cite{mensch_dictionary_2016}}\label{fig:compare_methods}
\end{figure*}

It remains to assess the performance of approximate code computation and
averaging techniques used in \somf. Indeed, subsampling for code
computation introduces noise that may undermine the computational speed-up.
To understand the impact of approximate code computation, we compare three
strategies to compute $(\balpha_t)_t$ on the HCP dataset. First, we
compute ${(\balpha_t^\star)}_t$ from ${(\x_t)}_t$ using~\eqref{eq:regression}. Subsampling is
thus used only in dictionary update. Second,
we rely on masked, non-consistent estimators~\eqref{eq:estimates}, as
in~\cite{mensch_dictionary_2016} --- this breaks convergence guarantees. Third,
we use averaged estimators $(\bbeta_t, \G_t)$ from~\eqref{eq:gram-estimates} to
reduce the variance in ${(\balpha_t)}_t$ computation.

%

Fig.~\ref{fig:compare_methods} compares the three strategies for $r\in \{12, 24 \}$. Partial minimization at each step is the
most important part to accelerate convergence: subsampling the dictionary updates already allows to outperforms
\omf. This is expected, as dictionary update constitutes the main
bottleneck of \omf in large-scale settings. Yet, for large reduction factors, using subsampling in code computation is important
to further accelerate convergence. This clearly appears when comparing the plain and
dashed black curves.
Using past estimates to better approximate~${(\balpha_t)}_t$ yields faster convergence
than the non-converging, masked loss strategy \eqref{eq:estimates} proposed in
\cite{mensch_dictionary_2016}.
%
%

\section{Conclusion}\label{conclusion}

In this paper, we introduce \somf, a matrix-factorization algorithm that
can handle input data with very large number of rows and columns. It leverages
subsampling within
the inner loop of a streaming algorithm to make iterations faster and
accelerate convergence. We show that \somf provides a stationary point of
the non-convex matrix factorization problem. To prove this result, we
extend the stochastic majorization-minimization framework to two major
approximations. We assess the performance of \somf on real-world large-scale problems, with
different sparsity/positivity requirements on learned factors. In particular, on fMRI and
hyperspectral data decomposition, we show that the use of subsampling can
speed-up decomposition up to $13$ times. The larger the dataset, the more \somf
outperforms state-of-the art techniques, which is very promising for future
applications. This calls for adaptation of our approach to learn more complex models.

%% file: appendices.tex


\section{Proofs of convergence}\label{app:demos}

This appendix contains the detailed proofs of Proposition~\ref{prop:approximate-surrogate} and Proposition~\ref{prop:somf}. We first introduce three lemmas that will be crucial to prove \samm convergence, before establishing it by proving Proposition~\ref{prop:approximate-surrogate}. Finally, we
show that \somf is indeed an instance of \samm (\textit{i.e.} meets the assumptions \ref{ass:compact}--\ref{ass:geometric}), proving Proposition~\ref{prop:somf}.

\subsection{Basic properties of the surrogates, estimate stability}

We derive an important result on the stability and optimality of the sequence $(\theta_t)_t$, formalized in Lemma~\ref{lemma:stability} --- introduced in the main text. We first introduce a numerical lemma on the boundedness of well-behaved determistic and random sequence.  The proof is detailed in Appendix~\ref{app:derivations}.

\begin{lemma}[Bounded quasi-geometric sequences]
    \label{lemma:bounded}
    Let~${(x_t)}_{t}$ be a sequence in $\RR^+$, $u: \RR \times \RR \to \RR$,
    $t_0 \in \mathbb{N}$ and $\alpha \in [0, 1)$ such that, for all $t \geq t_0, \, x_t \leq \alpha x_{t-1} + u(x_t, x_{t-1})$, where $u(x, y) \in o(x + y)$ for $x, y \to \infty$. Then ${(x_t)}_{t}$ is bounded.

    Let now $(X_t)_t$ be a random sequence in $\RR^+$, such that $\EE[X_t] < \infty$.
    We define ${(\mathcal{F}_{t})}_t$ the filtration adapted to ${(X_t)}_t$. If, for all $t > t_0$, there exists
    a $\sigma$-algebra $\mathcal{F}_{t'}$ such that $\mathcal{F}_{t-1} \subseteq \mathcal{F}_{t'} \subseteq \mathcal{F}_t$ and
    \begin{equation}
        \EE[X_t |\mathcal{F}_{t'} ] \leq \alpha X_{t-1} + u(X_t, X_{t-1}),
    \end{equation}
    then $(X_t)_t$ is bounded almost surely.
\end{lemma}

We first derive some properties of the approximate surrogate functions used in \samm. The proof is adapted from~\cite{mairal_stochastic_2013}.

\begin{lemma}[Basic properties of approximate surrogate functions]
    \label{lemma:basic}
    Consider any sequence of iterates ${(\theta_t)}_t$ and assume there exists
    $\epsilon > 0$ such that $g_t \in \mathcal{T}_{L, \rho}(f_t, \theta_{t-1},
    \epsilon)$ for all $t \geq 1$. Define $h_t \triangleq g_t - f_t$ for all $t \geq 1$,
    $\bar h_0 \triangleq h_0$ and
    $\bar h_t \triangleq (1 - w_t) \bar h_{t-1} + w_t h_t$.
    Under assumptions~\ref{ass:f_lip}~--~\ref{ass:weights_decay},

    \begin{enumerate}[(i)]
        \item $(\nabla {h_t(\theta_{t-1}))}_{t > 0}$ is uniformly bounded and there exists $R'$ such that ${\{\nabla h_t\}}_{t}$
        is uniformly bounded by $R'$.
        \item ${(h_t)}_t$ and ${(\bar h_t)}_t$ are uniformly $R'$-Lipschitz, ${(g_t)}_t$ and ${(\bar g_t)}_t$ are uniformly $(R + R')$-Lipschitz.
    \end{enumerate}
\end{lemma}

\begin{proof}
    We first prove (i). We set $\alpha > 0$ and define $\theta' = \theta_t - \alpha \frac{\nabla h_t(\theta_t)}{\Vert \nabla
    h_t(\theta_t)\Vert_2}$. As $h_t$ has a $L$-Lipschitz gradient on $\RR^K$, using Taylor's inequality (see Appendix~\ref{app:derivations})
    \begin{align}
        \label{eq:nabla_h_t_bounded}
        h_t(\theta') &\leq h_t(\theta_t)
        - \alpha \Vert \nabla h_t(\theta_t) \Vert_2
        + \frac{L\alpha^2}{2} \\
        \Vert \nabla h_t(\theta_t) \Vert_2
        &\leq \frac{1}{\alpha}(h_t(\theta_t) - h_t(\theta'))
        + \frac{L\alpha}{2}
        \leq \frac{2}{\alpha} \epsilon + \frac{L \alpha}{2}, \notag
    \end{align}
    where we use $h_t(\theta_t) < \epsilon$ and $-h_t(\theta_t') \leq \epsilon$ from the assumption $g_t \in \mathcal{T}_{L, \rho}(f_t, \theta_{t-1}, \epsilon)$.
    Moreover, by definition, $\nabla h_t$ exists and is $L$-lipschitz for all $t$.
    Therefore, $\forall\,t \geq 1$,
    \begin{equation}
        \Vert \nabla h_t(\theta) \Vert_2
        \leq \Vert \nabla h_t(\theta_t)\Vert_2 + L {\Vert \theta_{t-1} - \theta \Vert}_2
    \end{equation}
    Since $\Theta$ is compact and ${({\Vert \nabla h_t(\theta_{t}) \Vert}_2)}_{t
    \geq 1}$ is bounded in~\eqref{eq:nabla_h_t_bounded}, $\nabla h_t$ is bounded by $R'$ independent of
    $t$. (ii) follows by basic considerations on Lipschitz functions.
\end{proof}

Finally, we prove a result on the stability of the estimates, that derives from combining the properties of $(g_t)_t$
and the geometric decrease assumption~\ref{ass:geometric}.

 \begin{lemma}[Estimate stability under \samm approximation]
 \label{lemma:stability}
    In the same setting as Lemma~\ref{lemma:basic}, with the additional assumption~\ref{ass:geometric} (expected linear decrease of $\bar g_t$ suboptimality), the
     sequence $\Vert \theta_t - \theta_{t-1} \Vert_2$ converges to $0$ as fast as ${(w_t)}_t$, and $\theta_t$ is asymptotically an exact minimizer. Namely, almost surely,
 \begin{equation}
     {\Vert \theta_t - \theta_{t-1} \Vert}_2 \in \bigO(w_t)\:\text{and}
     \: \bar g_t(\theta_t) - \bar g_t(\theta_t^\star) \in \bigO(w_t^2).
 \end{equation}
 \end{lemma}
 \addtocounter{lemma}{-1}

\begin{proof}

We first establish the result when a deterministic version
of~\ref{ass:geometric} holds, as it makes derivations simpler to follow.

\subsubsection{Determistic decrease rate}

We temporarily assume that decays are deterministic.
\begingroup
\setcounter{tmp}{\value{assumption}}
\renewcommand\theassumption{\textbf{(\Alph{stoass}$_{\det}$)}}
\begin{assumption}\label{ass:det_geometric}
    For all $t > 0$, $\bar g_t(\theta_t) < \bar g_t(\theta_{t-1})$.
    Moreover, there exists $\mu > 0$ such that, for all $t > 0$
    \begin{align}
        \bar g_t(\theta_t) - \bar g_t(\theta_t^\star) &\leq (1 -
        \mu) (\bar g_t(\theta_{t-1}) - \bar g_t(\theta_{t}^\star)) \notag\\
        \textrm{where}\quad\theta_t^\star &= \argmin_{\theta \in \Theta} \bar
        g_t(\theta),
    \end{align}
\end{assumption}
\endgroup
\setcounter{assumption}{\thetmp}
We introduce the following auxiliary positive values, that we will seek to bound in the proof:
\begin{align}
    A_t \triangleq \Vert \theta_t - \theta_{t-1} \Vert_2,
    \quad B_t \triangleq \Vert \theta_t - \theta_t^\star \Vert_2, \notag \\
    C_t \triangleq \Vert \theta_t^\star - \theta_{t-1}^\star \Vert_2,
    \quad D_t \triangleq \bar g_t(\theta_t) - \bar g_t(\theta_t^\star).
\end{align}
Our goal is to bound $A_t$. We first relate it to $C_t$ and $B_t$ using convexity of $\ell_2$ norm:
\begin{equation}
    \label{eq:A_t_0}
    A_t^2 \leq 3 B_t^2 + 3 B_{t-1}^2 + 3 C_t^2.
\end{equation}
As $\theta_t^\star$ is the minimizer of $\bar g_t$,
by strong convexity of $(\bar g_t)_t$,
\begin{equation}
    \label{eq:B_t}
    \frac{\rho}{2} B_t^2 = \frac{\rho}{2} \Vert \theta_t - \theta_t^\star \Vert_2^2 \leq D_t,
\end{equation}
while we also have
\begin{align}
    &\frac{\rho}{2} \Vert \theta_t^\star - \theta_{t-1}^\star \Vert_2^2 \leq \bar g_t(\theta_{t-1}^\star) - \bar g_t(\theta_t^\star) \notag \\
    &\leq (1 - w_t) \big(\bar g_{t-1}(\theta_{t-1}^\star) - \bar g_{t-1}(\theta_t^\star)\big)
    \!+\!w_t\big(g_t(\theta_{t-1}^\star) - g_t(\theta_t^\star)\big) \notag \\
    \label{eq:C_t}
    &\leq w_t (R + R') \Vert \theta_t^\star - \theta_{t-1}^\star \Vert_2,\:\text{and thus}\:
    C_t \leq w_t \frac{2 Q}{\rho}.
\end{align}
The second inequalities holds because $\theta_{t-1}^\star$ is a minimizer of $\bar g_{t-1}$ and $g_t$ is $Q$-Lipschitz, where $Q \triangleq R + R'$, using Lemma~\ref{lemma:basic}. Replacing~\eqref{eq:B_t} and~\eqref{eq:C_t} in~\eqref{eq:A_t_0} yields
\begin{equation}
    \label{eq:A_t}
    A_t^2 \leq \frac{6}{\rho} (D_t + D_{t-1}) + \frac{12 Q^2}{\rho}w_t^2,
\end{equation}
and we are left to show that $D_t \in \bigO(w_t^2)$ to conclude. For this, we decompose the inequality from~\ref{ass:det_geometric} into
\begin{align}
    &D_t \leq (1 - \mu)(
    \bar g_t(\theta_{t-1}) - \bar g_t(\theta_t^\star)
    ) \notag \\
    &= (1-\mu) \Big(
    w_t \big(g_t(\theta_{t-1}) - g_t(\theta_t)\big)
    + w_t \big(g_t(\theta_t) - g_t(\theta_t^\star)\big)
    \Big)
    \notag \\
    &\phantom{=}
    + (1-\mu) \Big((1 - w_t) \big(\bar g_{t-1}(\theta_{t-1}) - \bar g_{t-1}(\theta_{t-1}^\star)\big) \notag \\
    &\phantom{=+(1-\mu)\Big(}+ (1 - w_t) \big(\bar g_{t-1}(\theta_{t-1}^\star) - \bar g_{t-1}(\theta_t^\star)\big)
    \Big) \notag \\
    \label{eq:D_t}
    &\leq (1 - \mu) (w_t Q (A_t + B_t) + D_{t-1}),
\end{align}
where the second inequality holds for the same reasons as in~\eqref{eq:C_t}. Injecting~\eqref{eq:B_t} and~\eqref{eq:A_t} in~\eqref{eq:D_t}, we obtain
\begin{equation}
    \label{eq:D_t_simp}
    \tilde D_t \leq (1 - \mu) \tilde D_{t-1} \frac{w_{t-1}^2}{w_t^2}
    + u(\tilde D_t, \tilde D_{t-1}),
\end{equation}
where we define $\tilde D_t \triangleq \frac{D_t}{w_t^2}$. It is easy to show
(see algebraic details in Appendix~\ref{app:derivations}) that the
perturbation term $u(\tilde D_t, \tilde D_{t-1}) \in o(\tilde D_t + \tilde
D_{t-1})$ if $\tilde D_t \to \infty$. Using the determistictic result of
Lemma~\ref{lemma:bounded},
this ensures that
$\tilde D_t$ is bounded, which combined with~\eqref{eq:B_t} allows to conclude.

\subsubsection{Stochastic decrease rates}In the general case~\ref{ass:geometric}, the inequalities~\eqref{eq:B_t},~\eqref{eq:C_t} and~\eqref{eq:A_t} holds, and \eqref{eq:D_t_simp} is replaced by
\begin{equation}
    \label{eq:exp_D_t}
    \EE[\tilde D_t | \mathcal{F}_{t-\frac{1}{2}} ] \leq (1 - \mu) \tilde D_{t-1} \frac{w_{t-1}^2}{w_t^2}
    + u(\tilde D_t, \tilde D_{t-1}),
\end{equation}
Taking the expectation of this inequality and using Jensen inequality,
we show that~\eqref{eq:D_t} holds when replacing $\tilde D_t$ by $\EE[\tilde D_t]$. This shows that $\EE[D_t] \in \bigO(w_t^2)$ and thus $\EE[D_t] < \infty$. The result follows from Lemma~\ref{lemma:bounded}, that applies as $\mathcal{F}_{t-1} \subseteq \mathcal{F}_{t-\frac{1}{2}} \subseteq \mathcal{F}_{t}$.
\end{proof}

\subsection[SAMMM convergence]{Convergence of \samm~--- Proof of Proposition~\ref{prop:approximate-surrogate}}

We now proceed to prove the Proposition~\ref{prop:approximate-surrogate}, that
extends the stochastic majorization-minimization framework to allow approximations
in both majorization and minimizations steps.

\begin{proof}[Proof of Proposition~\ref{prop:approximate-surrogate}]
We adapt the proof of Proposition 3.3 from~\cite{mairal_stochastic_2013} (reproduced as Proposition~\ref{prop:param} in our work).
Relaxing tightness and majorizing hypotheseses introduces some extra error terms in
the derivations. Assumption~\ref{ass:epsilon_decay} allows to control these
extra terms without breaking convergence. The stability Lemma~\ref{lemma:stability} is important in steps 3 and 5.

\subsubsection[Convergence of aggregated surrogate]{Almost sure convergence of $(\bar g_t(\theta_t))$}

We control the positive expected variation of
${(g_t(\theta_t))}_t$ to show that it is a converging quasi-martingale. By
construction of $\bar g_t$ and properties of the surrogates $g_t  \in \mathcal{T}_{\rho, L}(f_t, \theta_{t-1}, \epsilon_t)$, where $\epsilon_t$ is
a non-negative sequence that meets~\ref{ass:epsilon_decay},
\begin{align}
    \label{eq:key_maj}
    &\bar g_t(\theta_t) - \bar g_{t-1}(\theta_{t-1}) \notag \\
    &= (\bar g_t(\theta_t) - \bar g_t(\theta_{t-1}))
    + w_t(g_t(\theta_{t-1}) - \bar g_{t-1}(\theta_{t-1})) \notag \\
    &\leq w_t(g_t(\theta_{t-1}) - \bar g_{t-1}(\theta_{t-1})) \notag \\
    &\leq w_t(g_t(\theta_{t-1}) - f_t(\theta_{t-1}))
    + w_t(f_t(\theta_{t-1}) - \bar f_{t-1}(\theta_{t-1})) \notag \\
    &\phantom{=} + w_t(\bar f_{t-1}(\theta_{t-1}) - \bar g_{t-1}(\theta_{t-1})) \notag \\
    &\leq w_t(f_t(\theta_{t-1}) - \bar f_{t-1}(\theta_{t-1})) + w_t (\bar \epsilon_{t-1} + \epsilon_t),
\end{align}
where the average error sequence ${(\bar \epsilon_t)_t}$ is defined recursively:
$\bar \epsilon_0 \triangleq \epsilon_0$ and $\bar \epsilon_t \triangleq ( 1 - w_t) \epsilon_{t-1} + w_t \epsilon_t$. The first inequality uses $\bar g_t(\theta_{t}) \leq \bar g_t(\theta_{t-1})$.
To obtain the forth inequality we observe $g_t(\theta_{t-1}) - f_t(\theta_{t-1})<\epsilon_t$ by definition of $\epsilon_t$ and $\bar f_t(\theta_{t-1}) - \bar g_t(\theta_{t-1})
 \leq \bar \epsilon_t$,
which can easily be shown by induction on $t$. Then, taking the conditional expectation
with respect to~$\mathcal{F}_{t-1}$,
\begin{align}
    \label{eq:positive_decomposition}
&\EE [\bar g_t(\theta_t) - \bar g_{t-1}(\theta_{t-1}) | \mathcal{F}_{t-1}] \notag \\
&\leq w_t \sup_{\theta \in \Theta} \vert f(\theta) - \bar f_{t-1}(\theta) \vert
 + w_t (\bar \epsilon_{t-1} + \EE [\epsilon_t | \mathcal{F}_{t-1}]).
\end{align}
We have used the fact that $\epsilon_{t-1}$ is
deterministic with respect to $\mathcal{F}_{t-1}$. To ensure convergence, we must
bound both terms in~\eqref{eq:positive_decomposition}: the first term is the same
as in the original proof with exact surrogate, while the second is the perturbative term
introduced by the approximation sequence ${(\epsilon_t)}_t$.
We use Lemma B.7 from~\cite{mairal_stochastic_2013}, issued from the theory of empirical processes:
$\EE[\sup_{\theta \in \Theta} \vert f(\theta) - \bar f_{t-1}(\theta) \vert]
= \mathcal{O}(w_t t^{1/2})$, and thus
\begin{equation}
    \label{eq:empirical_process}
    \sum
    _{t=1}^\infty w_t \EE[\sup_{\theta \in \Theta}
    \vert f(\theta) - \bar f_{t-1}(\theta) \vert]
    < C \sum _{t=1}^\infty t^{1/2}w_t^2 < \infty
\end{equation}
where $C$ is a constant, as $t^{1/2} w_t^2 = t^{1/2 - 2u}$ and $u > 3/4$ from~\ref{ass:weights_decay}.
Let us now focus on the second term of~\eqref{eq:positive_decomposition}.
Defining, for all $1 \leq i \leq t$, $w_i^t = w_i \prod_{j=i+1}^t (1 - w_j)$,
\begin{equation}
    \EE [\bar \epsilon_t] = \sum_{i=1}^t w_i^t \EE [\epsilon_t]
    \leq w_t \sum_{i=1}^t \EE[\epsilon_t].
\end{equation}
We set $\eta > 0$ so that $2(u-1) - \eta > -1$. Assumption~\ref{ass:epsilon_decay} ensures $\EE [\epsilon_t] \in
\mathcal{O}(t^{2(u-1) - \eta})$, which allows to bound the partial sum $\sum_{i=1}^t \EE [\epsilon_i] \in
\mathcal{O}(t^{2u - 1 - \eta})$. Therefore
\begin{align}
    \label{eq:epsilon_convergence}
        &w_t \EE [\bar \epsilon_{t-1} + \EE[\epsilon_{t} | \mathcal{F}_{t-1}]]
        = w_t \EE[\epsilon_{t-1}] + w_t \EE[\epsilon_t] \notag \\
        &\leq w_t^2 \Big(\sum_{i=1}^t \EE[\epsilon_t] \Big) + w_t \EE[\epsilon_{t}] \\
        &\leq A t^{2u - 2u - 1 - \eta} + B t^{2u - u - 2 - \eta} \leq C t^{- 1 - \eta}, \notag
\end{align}
where we use $u < 1$ on the third line and the definition of ${(w_t)}_t$ on the second line.
Thus $\sum_{t=1}^\infty w_t \EE [\bar \epsilon_{t-1} +\EE[\epsilon_{t} | \mathcal{F}_{t-1}]] < \infty$.
We use quasi-martingale theory to conclude, as in~\cite{mairal_stochastic_2013}. We define the variable $\delta_t$ to be
$1$ if $\EE[\bar g_t(\theta_t) - \bar g_{t-1}(\theta_{t-1}) | \mathcal{F}_{t-1}] \geq 0$, and~$0$ otherwise. As all terms of~\eqref{eq:positive_decomposition} are positive:
\begin{align}
    &\sum_{t=1}^\infty \EE[\delta_t(\bar g_t(\theta_t) -\bar g_{t-1}(\theta_{t-1}))] \notag \\
    &= \sum_{t=1}^\infty \EE[\delta_t \EE[\bar g_t(\theta_t)
    - \bar g_{t-1}(\theta_{t-1}) | \mathcal{F}_{t-1}]] \\
    &\leq \sum_{t=1}^\infty w_t \EE[\sup_{\theta \in \Theta}
    \vert f(\theta) - \bar f_{t-1}(\theta) \vert + \bar \epsilon_{t-1} + \EE[\epsilon_t | \mathcal{F}_{t-1}]
    \vert] < \infty \notag .
\end{align}
As $\bar g_t$ are bounded from below ($\bar f_t$ is bounded from~\ref{ass:f_lip} and we easily show that $\bar \epsilon_t$ is bounded), we can apply Theorem A.1 from~\cite{mairal_stochastic_2013}, that is a quasi-martingale convergence
theorem originally found in~\cite{metivier_semimartingales:_1982}. It ensures that ${(g_t(\theta_t))}_{t \geq
1}$ converges almost surely to an integrable random variable $g^\star$, and that
$\sum_{t=1}^\infty \EE[ | \EE[\bar g_t(\theta_t) - \bar g_{t-1}(\theta_{t-1}) |
\mathcal{F}_{t-1} ] | ] < \infty$ almost surely.

\subsubsection[Convergence of aggregated loss]{Almost sure convergence of $\bar f(\theta_t)$}

We rewrite the second inequality of~\eqref{eq:key_maj}, adding $\bar \epsilon_t$ on both sides:
\begin{align}
    \label{eq:key_maj_rev}
    0 &\leq w_t\big(\bar g_{t-1}(\theta_{t-1}) - \bar f_{t-1}(\theta_{t-1})+ \bar \epsilon_{t-1} \big) \notag \\
    &\leq w_t\big(g_t(\theta_{t-1})-f_t(\theta_{t-1})\big) + w_t\big(f_t(\theta_{t-1}) - \bar f_{t-1}(\theta_{t-1})\big) \notag \\
    &\phantom{=}
    + \big(\bar g_{t-1}(\theta_{t-1}) - \bar g_t(\theta_t)\big) + w_t \bar \epsilon_{t-1} \notag \\
    &\leq w_t\big(f_t(\theta_{t-1}) - \bar f_{t-1}(\theta_{t-1})\big)
    + \big(\bar g_{t-1}(\theta_{t-1}) - \bar g_t(\theta_t)\big) \notag \\
     &\phantom{=}+w_t(\epsilon_t + \bar \epsilon_{t-1}),
\end{align}
where the left side bound has been obtained in the last paragraph by induction and the right side
bound arises from the definition of $\epsilon_t$.
Taking the expectation of~\eqref{eq:key_maj_rev} conditioned on $\mathcal{F}_{t-1}$,
almost surely,
\begin{align}
    \label{eq:key_maj_exp}
    0
    &\leq
    w_t (f(\theta_{t-1})
    - \bar f_{t-1} (\theta_{t-1})) \\
    &\phantom{=} - \EE [\bar g_t(\theta_t)
    - \bar g_{t-1}(\theta_{t-1}) | \mathcal{F}_{t-1}]
    + w_t (\bar \epsilon_{t-1}
    + \EE [\epsilon_t | \mathcal{F}_{t-1}]), \notag
\end{align}
We separately study the three terms of the previous upper bound. The first two terms can undergo the same analysis as in~\cite{mairal_stochastic_2013}.
First, almost sure convergence of $\sum_{t=1}^\infty \EE\big[ |
\EE[\bar g_t(\theta_t) - \bar g_{t-1}(\theta_{t-1}) | \mathcal{F}_{t-1} ] | \big]$
implies that $\EE\big[\bar g_t(\theta_t) - \bar g_{t-1} (\theta_{t-1})|
\mathcal{F}_{t-1}\big]$ is the summand of an almost surely converging sum.
Second, $w_t
\big(f(\theta_{t-1}) - \bar f_{t-1}(\theta_{t-1})\big)$ is the summand of an absolutely
converging sum with probability one, less it would contradict~\eqref{eq:empirical_process}.
To bound the third term, we have once more to control the perturbation introduced
by ${(\epsilon_t)}_t$. We have $\sum_{t=1}^\infty w_t \bar
\epsilon_{t-1} + w_t \EE[\epsilon_t | \mathcal{F}_{t-1}] <
\infty$ almost surely, otherwise Fubini's theorem would invalidate~\eqref{eq:epsilon_convergence}.

As the three terms are the summand of absolutely converging sums, the positive
term $w_t(\bar g_{t-1}(\theta_{t-1}) - \bar f_{t-1}(\theta_{t-1}) + \bar
\epsilon_{t-1})$ is the summand of an almost surely convergent sum. This is not
enough to prove that $\bar h_t(\theta_t) \triangleq \bar g_t(\theta_t) - \bar
f_t(\theta_t) \to_\infty 0$, hence we follow~\cite{mairal_stochastic_2013} and
make use of its Lemma A.6. We define $X_t \triangleq \bar h_{t-1}(\theta_{t-1}) + \bar
\epsilon_{t-1}$. As~\ref{ass:epsilon_decay} holds, we use
Lemma~\ref{lemma:stability}, which ensures that ${(\bar h_t)}_{t\geq 1}$ are uniformly
$R'$-Lipschitz and $\Vert \theta_t-\theta_{t-1} \Vert_2 = \mathcal{O}(w_t)$.
Hence,
\begin{align}
    &\vert X_{t+1} - X_t \vert
    \leq \vert \bar h_t(\theta_t) - \bar h_{t-1}(\theta_{t-1}) \vert
    + \vert \bar \epsilon_{t} - \bar \epsilon_{t-1} \vert \notag \\
    &\leq R' \Vert \theta_t - \theta_{t-1} \Vert_2 + \vert \bar \epsilon_{t} - \bar \epsilon_{t-1} \vert,\quad\text{as $\bar h_t$ is $R'$-Lipschitz} \notag\\
    &\leq \mathcal{O}(w_t) + \vert \bar \epsilon_{t} - \bar \epsilon_{t-1} \vert, \quad\text{as }\Vert \theta_t -\theta_{t-1} \Vert_2 = \bigO(w_t)
\end{align}
From assumption~\ref{ass:epsilon_decay}, $(\epsilon_t)_t$ and $(\bar \epsilon_t)_t$ are bounded. Therefore
$\vert \bar \epsilon_{t} - \bar \epsilon_{t-1} \vert
\leq w_t (\vert \epsilon_t \vert + \vert \bar \epsilon_{t-1} \vert)
\in \mathcal{O}(w_t)$ and hence
\begin{equation}
\vert X_{t+1} - X_t \vert \leq \mathcal{O}(w_t).
\end{equation}

Lemma A.6 from~\cite{mairal_stochastic_2013} then ensures
that $X_t$ converges to zero with probability one. Assumption~\ref{ass:epsilon_decay} ensures that $\epsilon_t \to_\infty~0$ almost surely, from which we can easily deduce $\bar \epsilon_t \to_\infty 0$ almost surely. Therefore $\bar h_t(\theta_t) \to 0$ with probability one and ${(\bar
f_t(\theta_t))}_{t\geq 1}$ converges almost surely to $g^\star$.

\subsubsection[Convergence of empirical loss]{Almost sure convergence of $\bar f(\theta_t)$}Lemma B.7 of~\cite{mairal_stochastic_2013}, based on empirical process theory~\cite{van_der_vaart_asymptotic_2000}, ensures that $\bar f_t$ uniformly converges to
$\bar f$. Therefore, ${(\bar f(\theta_t))}_{t\geq 1}$ converges almost surely to $g^\star$.

\subsubsection{Asymptotic stationary point condition}

Preliminary to the final result, we establish the asymptotic stationary point condition~\eqref{eq:asymptotic_stat} as in~\cite{mairal_stochastic_2013}.
This requires to adapt the
original proof to take into account the errors
in surrogate computation and minimization.
We set $\alpha > 0$. By definition, $\nabla \bar h_t$ is $L$-Lipschitz over $\RR^K$.
Following the same computation as in~\eqref{eq:nabla_h_t_bounded}, we obtain, for all $\alpha > 0$,
\begin{equation}
    \label{eq:ineq_bar_ht}
    \Vert \nabla \bar h_t(\theta_t) \Vert_2 \leq \frac{2}{\alpha} \bar \epsilon_t + \frac{L \alpha}{2},
\end{equation}
where we use $ \vert \bar h_t(\theta) \vert \leq \bar \epsilon_t$ for all $\theta \in \RR^K$.
As $\bar \epsilon_t \to 0$ and the inequality~\eqref{eq:ineq_bar_ht} is true for all $\alpha$,
$\Vert \nabla \bar h_t(\theta_t) \Vert_2 \to_\infty 0$ almost surely. From the strong
convexity of $\bar g_t$ and Lemma~\ref{lemma:stability}, $\Vert \theta_t - \theta_t^\star \Vert_2$ converges to zero, which ensures
\begin{equation}
    \label{eq:asymptotic_stat}
    \Vert \nabla \bar h_t(\theta_t^\star) \Vert_2
    \leq \Vert \bar \nabla h_t(\theta_t) \Vert_2
    + L \Vert \theta_t - \theta_t^\star \Vert_2 \to_\infty 0.
\end{equation}
\subsubsection{Parametrized surrogates}

We use assumption~\ref{ass:param} to finally prove the property, adapting the proof of
Proposition 3.4 in~\cite{mairal_stochastic_2013}. We first recall the derivations of~\cite{mairal_stochastic_2013} for obtaining~\eqref{eq:infty_der}
We define $(\bkappa_t)_t$ such that $\bar g_t = g_{\bkappa_t}$ for all $t > 0$.
We assume that $\theta_\infty$ is a limit point of ${(\theta_t)}_t$. As
$\Theta$ is compact, there exists an increasing sequence $(t_k)_k$ such that
$(\theta_{t_k})_k$ converges toward $\theta_\infty$. As $\mathcal{K}$ is
compact, a converging subsequence of $(\bkappa_{t_k})_k$ can be extracted, that
converges towards $\bkappa_\infty \in \mathcal{K}$. From the sake of
simplicity, we drop subindices and assume without loss of generality that
$\theta_t \to \theta_\infty$ and $\bkappa_t \to \bkappa_\infty$. From the
compact parametrization assumption, we easily show that ${(\bar
g_{\bkappa_t})}_t$ uniformly converges towards $\bar g_\infty \triangleq \bar
g_{\bkappa_\infty}$. Then, defining $\bar h_\infty = \bar g_\infty - \bar f$,
for all $\theta \in \Theta$,
\begin{equation}
    \label{eq:infty_der}
    \nabla \bar f(\theta_\infty,\theta - \theta_\infty)
    = \nabla \bar g_\infty(\theta_\infty, \theta - \theta_\infty)
    - \nabla \bar h_\infty(\theta_\infty, \theta - \theta_\infty)
\end{equation}
We first show that $\nabla \bar f(\theta_\infty,\theta - \theta_\infty) \geq 0$ for all $\theta \in \Theta$. We consider the sequence ${(\theta_t^\star)}_t$. From Lemma~\ref{lemma:stability},
$\Vert \theta_t - \theta_t^\star \Vert_2 \to 0$, which implies $\theta_t^\star \to \theta_\infty$.
$\bar g_t$ converges uniformly towards $\bar g_\infty$, which implies ${(\bar g_t(\theta_t^\star))}_t
\to \bar g_\infty(\theta_\infty)$. Furthermore, as $\theta_t^\star$ minimizes $\bar g_t$, for all $t > 0$ and $\theta \in \Theta$, $\bar g_t(\theta_t^\star) \leq \bar g_t(\theta)$. This implies
$\bar g_\infty(\theta_\infty) \leq \inf_{\theta \in \Theta} \bar g_\infty(\theta)$
by taking the limit for $t \to \infty$. Therefore $\theta_\infty$ is the minimizer
of $\bar g_\infty$ and thus $\nabla \bar g_\infty(\theta_\infty,  \theta - \theta_\infty)
\geq 0$.

Adapting~\cite{mairal_stochastic_2013}, we perform the first-order expansion of $\bar h_t$ around $\theta_t^\star$ (instead of $\theta_t$ in the original proof)
and show that $\nabla \bar h_\infty(\theta_\infty, \theta - \theta_\infty) =
0$, as $\bar h_t$ differentiable, $\Vert \nabla \bar
h_t(\theta_t^\star) \Vert_2 \to 0$ and $\theta_t^\star \to \theta_\infty$. This is sufficient to conclude.
\end{proof}

\subsection{Convergence of \somf~--- Proof of Proposition~\ref{prop:somf}}
\begin{proof}[Proof of Proposition~\ref{prop:somf}]From assumption~\ref{ass:f_lip}, ${(x_t)}_t$ is
$\ell_2$-bounded by a constant $X$. With assumption \ref{ass:D_cond},
it implies that ${(\balpha_t)}_t$ is $\ell_2$-bounded by a
constant $A$. This is enough to show that $(g_t)_t$ and $(\theta_t)_t$ meet basic
assumptions \ref{ass:compact}--\ref{ass:param}. Assumption~\ref{ass:weights_decay} immediately implies~\ref{ass:gamma_decay}. It remains to show that $(g_t)_t$ and
$(\theta_t)_t$ meet the assumptions \ref{ass:epsilon_decay} and
\ref{ass:geometric}. This will allow to cast \somf as an instance of \samm and conclude.

\subsubsection{The computation of $\D_t$ verifies \ref{ass:geometric}}

We define $\D_t^\star = \argmin_{\D \in \mathcal{C}} \bar g_t(\D)$. We show that
performing subsampled block coordinate descent on $\bar g_t$ is sufficient to meet assumption~\ref{ass:geometric}, where $\theta_t = \D_t$.
We separately analyse the exceptional case where no subsampling is done and the general case.

First, with small but non-zero
probability,~$\M_t = \I_p$ and Alg.~\ref{alg:somf-dictionary} performs a
single pass of simple block coordinate descent on $\bar g_t$. In this case, as
$\bar g_t$ is strongly convex
from~\ref{ass:D_cond},~\cite{beck_convergence_2013, wright_coordinate_2015}
ensures that the sub-optimality decreases at least of factor $1 - \mu$ with a single pass of block coordinate descent, where
$\mu > 0$ is a constant independent of~$t$. We provide an explicit $\mu$ in Appendix~\ref{app:derivations}.

In the general case, the function value decreases deterministically at each
minimization step: $\bar g_t(\D_t) \leq \bar g_t(\D_{t-1})$.  As a consequence,
$\EE[\bar g_t(\D_t) | \mathcal{F}_{t-\frac{1}{2}}, \M_t \neq \I_p] \leq \bar
g_t(\D_{t-1})$. Furthermore, $\bar g_t$ and hence $\bar
g_t(\D_t^\star)$ are deterministic with respect to $\mathcal{F}_{t-\frac{1}{2}}$, which implies
$\EE[\bar g_t(\D_t^\star) | \mathcal{F}_{t-\frac{1}{2}}, \M_t \neq \I_p] = \bar
    g_t(\D_t^\star)$. Defining $d \triangleq \PP[\M_t = \I_p]$, we split the sub-optimality expectation and combine the analysis of both cases:
\begin{align}
    &\EE[\bar g_t(\D_t) - \bar g_t(\D_t^\star) | \mathcal{F}_{t-\frac{1}{2}}] \notag \\
    &= d \EE[\bar g_t(\D_t) - \bar g_t(\D_t^\star) | \mathcal{F}_{t-\frac{1}{2}}, \M_t = \I_p] \notag \\
    &\phantom{=}+ (1 - d) \EE[\bar g_t(\D_t) - \bar g_t(\D_t^\star) | \mathcal{F}_{t-\frac{1}{2}}, \M_t \neq \I_p] \notag \\
    &\leq \big(d (1 - \mu) + (1 -d)\big)(\bar g_t(\D_{t-1}) - \bar g_t(\D_t^\star))\notag \\
    &= \big(1 - d \mu \big) (\bar g_t(\D_{t-1}) - \bar g_t(\D_t^\star)).
\end{align}%
\subsubsection[Approximate surrogates]{The surrogates ${(g_t)}_t$ verify \ref{ass:epsilon_decay}} We define $g_t^\star \in \mathcal{S}_{\rho, L}(f_t, \D_{t-1})$ the surrogate used in \omf
at iteration $t$, which depends on the \textit{exact} computation of $\balpha_t^\star$, while the surrogate $g_t$ used in \somf relies on approximated $\balpha_t$. Formally,
using the loss function
$\ell(\balpha, \G, \bbeta) \triangleq
    \frac{1}{2} \balpha^\top \G \balpha - \balpha^\top \bbeta + \lambda \Omega (\balpha)$, we recall the definitions
\begin{gather}
    \balpha_t^\star {\triangleq} \argmin_{\balpha \in \RR^k} \ell(\balpha, \G_t^\star, \bbeta_t^\star),\,
    \balpha_t {\triangleq} \argmin_{\balpha \in \RR^k} \ell(\balpha, \G_t, \bbeta_t), \\
    g_t^\star(\D) \triangleq \ell(\balpha_t^\star, \D^\top \D, \D^\top \x_t),\:
    g_t(\D) \triangleq \ell(\balpha_t, \D^\top \D, \D^\top \x_t). \notag
\end{gather}
The matrices $\G_t^\star$, $\bbeta_t^\star$ are defined in~\eqref{eq:regression} and $\G_t$, $\bbeta_t$ in either the update rules~\eqref{eq:agg-estimates} or \eqref{eq:gram-estimates}.
We define $\epsilon_t \triangleq \Vert g_t^\star - g_t \Vert_\infty$ to be the $\ell_\infty$ difference between
the approximate surrogate of \somf and the exact surrogate of \omf, as illustrated in Figure~\ref{fig:samm}. By definition, $g_t \in \mathcal{T}_{\rho, L}(f_t, \theta_{t-1}, \epsilon_t)$. We first show that $\epsilon_t$ can be bounded by the Froebenius distance between the approximate parameters $\G_t$,~$\bbeta_t$
and the exact parameters $\G_t^\star, \bbeta_t^\star$. Using Cauchy-Schwartz inequality, we first show that there exists a constant $C' > 0$ such that
for all $\D \in \mathcal{C}$,
\begin{equation}
    \label{eq:norm_gt}
    \vert g_t(\D) - g_t^\star(\D) \vert \leq C'
     \Vert \balpha_t - \balpha_t^* \Vert_2.
\end{equation}

Then, we show that the distance ${\Vert \balpha_t - \balpha_t^* \Vert}_2$ can itself be bounded: there exists $C'' > 0$ constant such that
\begin{equation}
    \label{eq:norm_balpha}
    {\Vert \balpha_t - \balpha_t^\star \Vert}_2
     \leq C'' ({\Vert \G_t^\star - \G_t \Vert}_F
     + {\Vert \bbeta_t^\star - \bbeta_t \Vert}_2).
\end{equation}
We combine both equations and take the supremum over $\D \in~\mathcal{C}$, yielding
\begin{equation}
    \label{eq:tilde_gt}
    \epsilon_t \leq C
    ({\Vert \G_t^\star - \G_t \Vert}_F
    + {\Vert \bbeta_t^\star - \bbeta_t \Vert}_2),
\end{equation}
where $C$ is constant. Detailed derivation of \eqref{eq:norm_gt} to \eqref{eq:tilde_gt} relies on assumption~\ref{ass:D_cond} and are reported in Appendix~\ref{app:derivations}.

In a second step, we show that $\Vert \G_t^\star - \G_t \Vert_F$ and $\Vert
 \bbeta_t^\star - \bbeta_t \Vert_2$ vanish almost surely, sufficiently fast.
 We focus on bounding $\Vert \bbeta_t - \bbeta_t^\star \Vert_2$ and proceed similarly
 for $\Vert \G_t - \G_t^\star \Vert_2$ when the update rules~\eqref{eq:agg-estimates} are used.
 For $t > 0$, we write $i \triangleq i_t$. Then
 \begin{equation*}
     \bbeta_t \triangleq \bbeta_t^{(i)}
     = \sum_{s \leq t, \x_s = \x^{(i)}} \gamma_{s,t}^{(i)} \D_{s-1}^\top \M_s \x^{(i)},
 \end{equation*}
where $\gamma_{s,t}^{(i)} = \gamma_{c^{(i)}_t} \prod_{s < t,
 \x_s = \x^{(i)}} (1 - \gamma_{c^{(i)}_s})$ and $c^{(i)}_t = \left|
 \left\lbrace s \leq t, \x_s = \x^{(i)} \right\rbrace \right|$. We can
 then decompose $\bbeta_t - \bbeta_t^\star$ as
 \begin{align}
     \label{eq:beta_maj}
     \bbeta_t - \bbeta_t^\star &= \sum_{s \leq t, \x_s = \x_t = \x^{(i)}}
      \gamma_{s,t}^{(i)} (\D_{s-1} - \D_{t-1})^\top \M_s \x^{(i)} \notag \\
     &\phantom{=}
     + \D_{t-1}^\top \Big( \sum_{s \leq t, \x_s = \x^{i)}} \gamma_{s,t}^{(i)} \M_s - \I\Big) \x^{(i)}.
\end{align}

The latter equation is composed of two terms: the first one captures the
approximation made by using old dictionaries in the computation of ${(\bbeta_t)}_t$, while the second
captures how the masking effect is averaged out as the number of epochs increases. Assumption~\ref{ass:gamma_decay}
allows to bound both terms at the same time. Setting $\eta \triangleq \frac{1}{2} \min\big(v - \frac{3}{4}, (3u - 2) - v\big) > 0$, a tedious but elementary derivation indeed shows $\EE[\Vert\bbeta_t - \bbeta_t^\star \Vert_2] \in \mathcal{O}(t^{2(u-1) - \eta})$ and $\epsilon_t \to 0$ almost surely --- see Appendix~\ref{app:derivations}. The \somf algorithm therefore meets assumption~\ref{ass:epsilon_decay}
and is a convergent \samm algorithm. Proposition~\ref{prop:somf} follows.%
\end{proof}
%
%
%
%
%
%
%


%% file: derivations.tex

\section{Algebraic details}\label{app:derivations}

\subsection{Proof of Lemma~\ref{lemma:bounded}}
\begin{proof}We first focus on the deterministic case. Assume that ${(x_t)}_{t}$ is not bounded. Then there exists a subsequence
    of ${(x_t)}_{t}$ that diverges towards $+\infty$. We assume without loss of generality that ${(x_t)}_{t} \to \infty$.
    Then, $x_t + x_{t-1} \to \infty$ and for all $\epsilon > 0$, using the asymptotic bounds on $u$, there exists $t_1 \geq t_0$ such that
    \begin{align}
        \forall t \geq t_1,\, x_t &\leq \alpha x_{t-1} + \epsilon( x_t + x_{t-1} ) \notag \\\text{and therefore}\quad
         x_t &\leq \frac{\alpha + \epsilon}{1 - \epsilon} x_{t-1}.
    \end{align}
    Setting $\epsilon$ small enough, we obtain that $x_t$ is bounded by a geometrically decreasing sequence after $t_1$, and converges to $0$, which contradicts our hypothesis.
    This is enough to conclude.

    In the random case, we consider a realization of ${(X_t)}_t$ that is not bounded, and assumes without loss of generality that it diverges to $+\infty$.
    Following the reasoning above, there exists $\beta < 1$, $t_1 > 0$, such that for all $t > t_1$, $\EE[X_t | \mathcal{F}_{t'}] \leq \beta X_{t-1}$,
    where $\mathcal{F}_{t-1} \subseteq \mathcal{F}_{t'} \subseteq \mathcal{F}_{t}$.
    Taking the expectation conditioned on $\mathcal{F}_{t-1}$, $\EE[X_t | \mathcal{F}_{t-1}] \leq \beta X_{t-1}$, as $X_{t-1}$ is deterministic
    conditioned on $\mathcal{F}_{t-1}$. Therefore
    $X_t$ is a supermartingale beyond a certain time. As $\EE[X_t] < \infty$, Doob's forward convergence lemma on discrete martingales~\cite{doob1990stochastic} ensures that ${(X_t)}_t$ converges almost surely. Therefore
    the event $\{{(X_t)}_t\:\text{is not bounded}\}$ cannot happen on a set with non-zero probability, less it would lead to a contradiction. The lemma follows.%
\end{proof}%
\subsection{Taylor's inequality for $L$-Lipschitz continuous functions}

This inequality is useful in the demonstration of Lemma~\ref{lemma:basic} and Proposition~\ref{prop:approximate-surrogate}. Let $f : \Theta \subset \RR^K \to \RR$ be a function
with $L$-Lipschitz gradient. That is, for all $\theta, \theta' \in \Theta, {\Vert\nabla f(\theta) - \nabla f(\theta') \Vert}_2 \leq L {\Vert \theta - \theta' \Vert}_2$. Then, for all $\theta, \theta' \in \Theta$,
\begin{equation}
    f(\theta') \leq f(\theta) + \nabla f(\theta)^\top (\theta' - \theta) + \frac{L}{2} {\Vert \theta - \theta' \Vert}_2^2.
\end{equation}
\subsection{Lemma~\ref{lemma:stability}:
Detailed control of $D_t$ in~\eqref{eq:D_t_simp}}

Injecting~\eqref{eq:B_t} and~\eqref{eq:A_t} in~\eqref{eq:D_t}, we obtain
\begin{equation}
    \tilde D_t \leq (1 - \mu) \tilde D_{t-1} \frac{w_{t-1}^2}{w_t^2}
    + u(\tilde D_t, \tilde D_{t-1}),\quad\text{where}\quad
    u(\tilde D_t, \tilde D_{t-1}) \triangleq (1 - \mu) \tilde Q
    \bigg (\sqrt{3(\tilde D_t + \tilde D_{t-1} \frac{w_{t-1}^2}{w_t^2}
    ) + \tilde Q
    } + \sqrt{\tilde D_t}\,\bigg).
\end{equation}
From assumption~\ref{ass:weights_decay}, $\frac{w^2_{t-1}}{w^2_t} \to 1$, and we have, from elementary comparisons, that $u(\tilde D_t, \tilde D_{t-1}) \in o(\tilde D_t + \tilde D_{t-1})$ if ${D_t \to \infty}$. Using the determistictic result of
Lemma~\ref{lemma:bounded},
this ensures that
$\tilde D_t$ is bounded.

\subsection{Detailed derivations in the proof of Proposition~\ref{prop:somf}}

Let us first exhibit a scaler $\mu > 0$ independent of $t$, for which \ref{ass:geometric} is met

\subsubsection{Geometric rate for single pass subsampled block coordinate descent}
. For $\D^{(j)} \in \RR^{p \times k}$ any matrix with non-zero $j$-th column $\d^{(j)}$ and zero elsewhere
\begin{equation}
    \label{eq:expected_full_decrease}
    \nabla \bar g_t(\D + \D^{(j)}) - \nabla \bar g_t(\D + \D^{(j)}) = \bar \C_t[j , j] \d^{(j)}
\end{equation}
and hence $\bar g_t$ gradient has \textit{component Lipschitz constant} $L_j = \bar \C_t[j , j]$ for component $j$,
as already noted in \cite{mairal_online_2010}.
Using~\cite{wright_coordinate_2015} terminology, $\nabla \bar g_t$ has \textit{coordinate Lipschitz constant}
 $L_{\mathrm{max}} \triangleq \max_{0 \leq j < k} \bar \C_t[j, j] \leq \max_{t > 0, 0 \leq j < k} \balpha_t[j]^2 \leq A^2$,
 as $(\balpha_t)_t$ is bounded from \ref{ass:D_cond}. As a consequence, $\bar g_t$ gradient is also $L$-Lipschitz continuous,
 where \cite{wright_coordinate_2015} note that $L \leq \sqrt{k} L_{\mathrm{max}}$.
Moreover, $\bar g_t$ is strongly convex with strong convexity modulus $\rho > 0$ by hypothesis \ref{ass:D_cond}.
Then, \cite{beck_convergence_2013} ensures that after one cycle over the $k$ blocks
\begin{align}
    \EE[\bar g_t(\D_t) - \bar g_t(\D_t^\star) | \mathcal{F}_{t-1}, \M_t = \I_p]
    &\leq \big(1 - \frac{\rho}{2 L_{\mathrm{max}}(1 + k L^2 / L_{\mathrm{max}}^2)} \big) (\bar g_t(\D_{t-1}) - \bar g_t(\D_t^\star)) \notag \\
    &\leq \big(1 - \mu \big) (\bar g_t(\D_{t-1}) - \bar g_t(\D_t^\star))\quad\text{where}\quad\mu \triangleq \frac{\rho}{2 A^2(1 + k^2)}
\end{align}
\subsubsection{Controling $\epsilon_t$ from $(\G_t, \bbeta_t), (\G_t^\star, \bbeta_t^\star)$ --- Equations~\ref{eq:norm_gt}--\ref{eq:norm_balpha}}

We detail the derivations that are required to show that~\ref{ass:epsilon_decay} is met in the proof of \somf convergence. We first show that $(\balpha_t)_t$ is bounded. We choose $D > 0$ such that ${\Vert \d^{(j)} \Vert}_2 \leq D$ for all $j \in [k]$ and $\D \in \mathcal{C}$, and $X$ such that ${\Vert \x \Vert}_2 \leq X$ for all $\x \in \mathcal{X}$. From assumption~\ref{ass:D_cond}, using the second-order growth condition, for all $t > 0$,\pagebreak
\begin{align}
    \frac{\rho}{2} {\Vert \balpha_t - 0 \Vert}_2^2
    &\leq \lambda \Omega(0) -
    (\frac{1}{2} \balpha_t^\top \G_t \balpha_t - \balpha_t^\top \bbeta_t + \lambda \Omega(\balpha_t) \notag \\
    \frac{\rho}{2} {\Vert \balpha_t \Vert}_2^2 + \frac{1}{2} \balpha_t^\top  \G_t \balpha_t
    &\leq 0
    + {\Vert \balpha_t \Vert}_2 {\Vert \bbeta_t \Vert}_2,\quad\text{hence} \notag \\
    \rho {\Vert \balpha_t \Vert}_2^2 &\leq \sqrt{k} r D X {\Vert \balpha_t \Vert}_2,\quad\text{and therefore}\quad
     {\Vert \balpha_t \Vert}_2 \leq \frac{\sqrt{k} r D X}{\rho} \triangleq A.
\end{align}
We have successively used the fact that $\Omega(0) = 0$, $\Omega(\balpha_t) \geq 0$, and ${\Vert \bbeta_t \Vert}_2 \leq \sqrt{k} r D X$, which can be shown by a simple induction on the number of epochs. For all $t > 0$, from the definition of $\balpha_t$ and $\balpha_t^\star$, for all $\D \in \mathcal{C}$:
\begin{align}
    \vert g_t(\D) - g_t^\star(\D) \vert &=
    \Big\vert \frac{1}{2}\trace \D^\top \D (\balpha_t \balpha_t^\top - \balpha_t^\star {\balpha_t^\star}^\top)
    - (\balpha_t - \balpha_t^\star)^\top \D^\top \x_t
    \Big\vert \notag \\
    &\leq \frac{1}{2} {\Vert \D^\top \D \Vert}_F
    {\Vert \balpha_t \balpha_t^\top - \balpha_t^\star {\balpha_t^\star}^\top \Vert}_F
    + {\Vert \D \Vert}_F {\Vert \x_t \Vert}_2 {\Vert \balpha_t - \balpha_t^\star \Vert}_2 \notag \\
    &\leq (k D^2 A + \sqrt{k} D X) {\Vert \balpha_t - \balpha_t^\star \Vert}_2,
\end{align}
where we use Cauchy-Schwartz inequality and elementary bounds on the Froebenius norm for the first inequality, and use \linebreak $\balpha_t, \balpha_t^\star \leq A$, $\x_t \leq X$ for all $t > 0$ and $\d^{(j)} \leq D$ for all $j \in [k]$ to obtain the second inequality, which is~\eqref{eq:norm_gt}
in the main text.

We now turn to control ${\Vert \balpha_t - \balpha_t^\star \Vert}_2$. We adapt the proof of Lemma B.6 from~\cite{mairal_optimization_2013}, that states the lipschitz continuity of the minimizers of some parametrized functions. By definition,
\begin{equation}
    \balpha_t^\star = \argmin_{\balpha \in \RR^k} \ell(\balpha, \G_t^\star, \bbeta_t^\star) \qquad
    \balpha_t = \argmin_{\balpha \in \RR^k} \ell(\balpha, \G_t, \bbeta_t),
\end{equation}
Assumption \ref{ass:D_cond} ensures that $\G_t \succ \rho \I_k$, therefore we can write the second-order growth condition
\begin{align}
    \frac{\rho}{2} {\Vert \balpha_t - \balpha_t^\star \Vert}_2^2 &\leq
    \ell(\balpha_t, \G_t^\star, \bbeta_t^\star) -
    \ell(\balpha_t, \G_t, \bbeta_t) \notag \\
    \frac{\rho}{2} {\Vert \balpha_t - \balpha_t^\star \Vert}_2^2 &\leq
    \ell(\balpha_t^\star, \G_t, \bbeta_t) -
    \ell(\balpha_t^\star, \G_t^\star, \bbeta_t^\star),\quad\text{and therefore} \notag \\
    \rho {\Vert \balpha_t - \balpha_t^\star \Vert}_2^2
    &\leq p(\balpha_t) - p(\balpha_t^\star),\quad\text{where}\quad p(\balpha) \triangleq \ell(\balpha, \G_t, \bbeta_t) -
    \ell(\balpha_t, \G_t^\star, \bbeta_t^\star).
\end{align}
$p$ takes a simple form and can differentiated with respect to $\balpha$. For all $\balpha \in \RR^k$ such that ${\Vert \balpha \Vert}_2 \leq A$,
\begin{align}
    \label{eq:nabla_p}
    p(\balpha) &= \frac{1}{2} \balpha^\top (\G_t - \G_t^\star) \balpha - \balpha^\top (\bbeta_t - \bbeta_t^\star) \notag \\
    \nabla p(\balpha) &= (\G_t - \G_t^\star) \balpha - (\bbeta_t - \bbeta_t^\star) \notag \\
    {\Vert \nabla p(\balpha) \Vert}_2 &\leq A {\Vert \G_t - \G_t^\star \Vert}_F + {\Vert \bbeta_t - \bbeta_t^\star \Vert}_2 \triangleq L
\end{align}
Therefore $p$ is $L$-Lipschitz on the ball of size $A$ where $\balpha_t$ and $\balpha_t^\star$ live, and
\begin{align}
    \rho {\Vert \balpha_t - \balpha_t^\star \Vert}_2^2 &\leq L {\Vert \balpha_t - \balpha_t^\star \Vert}_2
    \notag \\
    {\Vert \balpha_t - \balpha_t^\star \Vert}_2 &\leq
    \frac{A}{\rho} {\Vert \G_t - \G_t^\star \Vert}_F + \frac{1}{\rho} {\Vert \bbeta_t - \bbeta_t^\star \Vert}_2,
\end{align}
which is \eqref{eq:norm_balpha} in the main text. The bound~\eqref{eq:tilde_gt} on $\epsilon_t$ immediately follows.

\subsubsection{Bounding ${\Vert \bbeta_t - \bbeta_t^\star \Vert}_2$
in equation~\eqref{eq:beta_maj}}

Taking the $\ell_2$ norm in~\eqref{eq:beta_maj}, we have
${\Vert \bbeta_t - \bbeta_t^\star \Vert}_2 \leq B L_t + C R_t$, where $B$ and $C$ are positive constants independent of $t$ and we introduce the terms
\begin{equation}
    \label{eq:lt_rt}
        L_t \triangleq \sum_{s \leq t, \x_s = \x_t = \x^{(i)}} \gamma_{s,t}^{(i)} {\Vert \D_{s-1} - \D_{t-1} \Vert}_F,\qquad
        R_t \triangleq \Big{\Vert \big(\sum_{s \leq t, \x_s = \x^{(i)}} \gamma_{s,t}^{(i)} \M_s\big) - \I \Big\Vert}_F.
\end{equation}

\paragraph{Conditioning on the sequence of drawn indices}We recall that ${(i_t)}_t$
is the sequence of indices that are used to draw ${(\x_t)}_t$
from ${\{\x^{(i)}\}}_i$, namely such that $\x_t = \x^{(i_t)}$. ${(i_t)}_t$
is a sequence of i.i.d random variables, whose law is uniform in $[1, n]$.
For each $i \in [n]$, we
define the increasing sequence ${(t_b^{(i)})}_{b > 0}$ that record the iterations at which sample $(i)$ is drawn, \textit{i.e.} such that $i_{t_b} = i$ for all $b > 0$. For $t > 0$,
we recall that $c_t^{(i)} > 0$ is the integer that counts the number of time sample $(i)$ has appeared in the algorithm, \textit{i.e.} $c_t^{(i)} = \max\,\{b > 0, t_b^{(i)} \leq t\}$. These notations will help us understanding the behavior of ${(L_t)}_t$ and
${(R_t)}_t$.

\paragraph{Bounding $R_t$}
The right term $R_t$ takes its value into sequences that are running average of masking matrices. Formally, $R_t = {\Vert \bar \M_t^{(i_t)} - \I \Vert}_F$, where we define for all $i \in [n]$,
\begin{align}
    \label{eq:recursion}
    \bar \M_t^{(i)} \triangleq \sum_{b = 1}^{c_t^{(i)}} \gamma_{t_b^{(i)}, t_c^{(i)}}^{(i)} \M_{t_b},\quad\text{ which follows the recursion}
    \quad
    \left\lbrace
    \begin{array}{lll}
    \bar \M_t^{(i)} &=& (1 - \gamma_{c^{(i)}_{t}}) \bar \M_{t-1}^{(i)} + \gamma_{c^{(i)}_{t}} \M_t\quad\text{if $i = i_t$} \\
    \bar \M_t^{(i)} &=& \M_{t-1}^{(i)}\quad\text{if $i \neq i_t$} \\
    \bar \M_0^{(i)} &=& 0\quad\text{for all $i \in [n]$}
    \end{array}\right.
\end{align}
When sampling a sequence of indices $(i_s)_{s > 0}$, the $n$ random matrix sequences ${[ {(\bar \M_t^{(i)})}_{t\leq0} ]}_{i \in [n]}$ follows the same probability law as the sampling is uniform. We therefore focus on controling ${(\bar \M_t^{(0)})}_t$.
For simplicity, we write $c_t \triangleq c_t^{(0)}$. When
$\EE[\cdot]$ is the expectation over the sequence of indices $(i_s)_s$,
\begin{align}
    \label{eq:simple_gamma}
    \EE[{\Vert \bar \M_t^{(0)} - \I \Vert}_F]^2 &\leq
    \EE\big[\sum_{j=1}^p (\bar \M_t^{(0)}[j, j] - 1) \big]
    = p \EE[(\bar \M_t^{(0)}[0, 0] - 1)] \notag \\
    &\leq C\,p{(c_t)}^{1/2} \gamma_{c_t} = C\,p {(c_t)}^{1/2 - v},\quad\text{where $C$ is a constant independent of $t$}.
\end{align}
We have simply bounded the Froebenius norm by the $\ell_1$ norm in the first inequality and used the fact that all coefficients $\M_t[j, j]$ follows the same
Bernouilli law for all $t > 0$, $j \in [p]$. We then used Lemma B.7 from~\cite{mairal_stochastic_2013} for the last inequality. This lemma applies
as $\M_t[0, 0]$ follows the recursion~\eqref{eq:recursion}. It remains to take the expectation of \eqref{eq:simple_gamma}, over all possible sampling trajectories~$(i_s)_{s > 0}$:
\begin{align}
    \label{eq:expected_Rt}
    \EE[R_t] &= \EE\big[\EE[R_t | (i_s)_s]\big] = \EE\big[\EE[{\Vert \M_t^{(i_t)} - \I \Vert}_F | (i_s)_s]\big] = \EE\big[\EE[{\Vert \M_t^{(0)} - \I \Vert}_F | (i_s)_s]\big]
    = \EE[{\Vert \M_t^{(0)} - \I \Vert}_F] \notag \\
    &= C p\EE[{(c_t)}^{1/2-v}] \leq C p\EE[{(c_t)}^{2(u-1) - \eta}].
\end{align}
The last inequality arises from the definition of $\eta \triangleq \frac{1}{2} \min\big(v - \frac{3}{4}, (3u - 2) - v\big)$, as follows. First, $\eta > 0$ as $u > \frac{11}{12}$. Then, we successively have
\begin{equation}
    \begin{split}
    \frac{5}{2} - 2u < \frac{2}{3} < \frac{3}{4},\quad\text{as $u > \frac{11}{12}$},\qquad
    v \geq \frac{3}{4} + 2 \eta > \frac{5}{2} - 2u + 2\eta,\\
    \frac{1}{2} - v < \frac{1}{2}
    - \frac{5}{2} + 2u - 2\eta = 2(u-1) - 2\eta < 2(u-1) - \eta,\quad\text{which allows to conclude.}
    \end{split}
\end{equation}
Lemma B.7 from~\cite{mairal_stochastic_2013} also ensures that $\M_t[0, 0] \to 1$ almost surely when $t \to \infty$. Therefore ${(\bar \M_t^{(0)} - \I})_t$ converges towards~$0$ almost surely, given any sample sequence $(i_s)_s$. It thus converges almost surely when \textit{all} random variables of the algorithm are considered.
This is also true for ${(\bar \M_t^{(i)} - \I)}_t$ for all $i \in [n]$ and hence for $R_t$.


\paragraph{Bounding $L_t$}

As above, we define $n$ sequences ${[{(L_t^{(i)})}_t]}_{i \in [n]}$, such that $L_t = L_t^{(i_t)}$ for all $t > 0$. Namely,
\begin{equation}
    L_t^{(i)} \triangleq \sum_{\substack{s \leq t,\\ \x_s = \x_t = \x^{(i)}}} \gamma_{s,t}^{(i)} {\Vert \D_{s-1} - \D_{t-1} \Vert}_F
    = \sum_{b = 1}^{c_t^{(i)}} \gamma_{t_b^{(i)}, t_{c_t^{(i)}}^{(i)}}^{(i)} {\big\Vert \D_{t_b -1} - \D_{t_{c_t^{(i)}}-1} \big\Vert}_F.
\end{equation}
Once again, the sequences $\big[{(L_t^{(i)})}_t\big]_{i}$ all follows the same distribution when sampling over sequence of indices $(i_s)_s$. We thus focus on bounding ${(L_t^{(0)})}_t$. Once again, we drop
the $(0)$ superscripts in the right expression for simplicity.
We set $\nu \triangleq 3u -2 - \eta$. From assumption~\ref{ass:gamma_decay} and the definition of~$\eta$, we have $v < \nu < 1$.
We split the sum in two parts, around index $d_t \triangleq c_t - \lfloor {(c_t)}^\nu \rfloor$, where $\lfloor \cdot \rfloor$ takes the integer part of a real number. For simplicity, we write $d \triangleq d_t$ and $c \triangleq c_t$ in the following.
\begin{align}
    \label{eq:lag_term}
    L_t^{(0)} = \sum_{b = 1}^{c} \gamma_{t_b, t_c} {\big\Vert \D_{t_b-1} - \D_{t_c-1} \big\Vert}_F
    &\leq 2 \sqrt{k}D \sum_{b=1}^{d} \gamma_{t_b, t_c}
    + \sum_{b=d + 1}^c \gamma_{t_b, t} \sum_{s = t_b - 1}^{t_c -1} w_s
    \triangleq 2 \sqrt{k}D L_{t, 1}^{(0)} + L_{t, 2}^{(0)}
\end{align}
On the left side, we have bounded ${\Vert \D_t \Vert}_F$ by $\sqrt{k} D$, where $D$ is defined in the previous section. The right part uses the bound on ${\Vert \D_s - \D_t \Vert}_F$
provided by Lemma~\ref{lemma:stability}, that applies here as~$\ref{ass:geometric}$ is met and~\eqref{eq:tilde_gt} ensures that ${({\Vert g_t - g_t^\star \Vert}_\infty)}_t$ is bounded.

We now study both $L_{t, 1}^{(0)}$ and $L_{t, 2}^{(0)}$. First, for all $t > 0$,
\begin{align}
    \label{eq:Lt_L}
    L_{t, 1}^{(0)} \triangleq&\sum_{b=1}^{d} \gamma_{t_b, t_c} = \sum_{b=1}^{d} \gamma_b \prod_{p=b+1}^{c} (1 - \gamma_p)
    \leq \sum_{b=1}^{d} \gamma_b {(1 - \gamma_c)}^{c-b} \notag \\
    &\leq \frac{(1 - \gamma_c)^{\lfloor c^\nu \rfloor}}{\gamma_c}
    \leq c^v \exp\big({\log(1-\frac{1}{c^v}) c^\nu}\big) \leq C' c^{v} \exp(c^{\nu - v}) \leq C {c}^{2(u - 1) - \eta} = C {(c_t)}^{2(u - 1) - \eta},
\end{align}
where $C$ and $C'$ are constants independent of $t$. We have used $\nu > v$ for the third inequality, which ensures that
$\exp\big({\log(1-\frac{1}{c^v})c^\nu}\big) \in \mathcal{O}({c^{\nu - v}})$.
Basic asymptotic comparison provides the last inequality, as $c_t \to \infty$ almost surely and the right term decays exponentially in ${(c_t)}_t$, while the left decays polynomially. As a consequence, $L_{t, 1}^{(0)} \to 0$ almost surely.


Secondly, the right term can be bounded as ${(w_t)}_t$ decays sufficiently rapidly. Indeed, as $\sum_{b=1}^c \gamma_{t_b, t} = 1$, we have
\begin{align}
    \label{eq:Lt_R}
    L_{t, 2}^{(0)} \triangleq \sum_{b=d}^c \gamma_{t_b, t} \sum_{s = t_b - 1}^{t_c -1} w_s &\leq \max_{d \leq b \leq c} \Big( \sum_{s = t_b - 1}^{t_c-1}  w_s \Big) \notag
    = \sum_{s = t_{d} - 1}^{t_c -1}  w_s \\
    &\leq w_{t_d} (t_c - t_d) = \frac{t_c - t_d}{{(t_d)}^u}
    = \frac{c_t - d_t}{{(d_t)}^u} \frac{t_c - t_d}{c_t - d_t} (\frac{d_t}{t_d})^u
\end{align}
from elementary comparisons. First, we use the definition of $\nu$ to draw
\begin{equation}
        \frac{c_t - d_t}{(d_t)^u} \leq \frac{(c_t)^\nu}{(c_t)^u(1 - c_t^{\nu - 1})^u}
    \leq C (c_t)^{\nu - u} = C (c_t)^{2(u-1) - \eta},
\end{equation}
were we use the fast that $\eta - 1 < 0$. We note that for all $b > 0$, $t_{b+1} - t_b$ follows a geometric law of
parameter $\frac{1}{n}$, and expectation~$n$. Therefore, as $c -d \to \infty$ when $t \to 0$, from the strong law of
large numbers and linearity of the expectation
\begin{equation}
    \label{eq:as_Lt_R}
    \frac{t_c - t_d}{c-d} = \frac{1}{c-d}\sum_{b=d}^{c-1} t_{b+1} - t_b \to  n, \qquad
     \frac{t_d}{d} = \frac{1}{d}\sum_{b=0}^{d-1} t_{b+1} - t_b \to n\quad\text{almost surely.}
\end{equation}
As a consequence, $\frac{t_c - t_d}{c_t - d_t} (\frac{d_t}{t_d})^u \to n^{1 -u}$ almost surely. This immediately shows $L_{t, 2}^{(0)} \to 0$ and thus $L_t^{(0)} \to 0$ almost surely. As with $R_t$, this implies that $L_t \to 0$ almost surely and therefore
\begin{equation}
    {\Vert \bbeta_t - \bbeta_t^{\star} \Vert}_2 \to 0\quad{\text{almost surely.}}
\end{equation}
Finally, from the dominated convergence theorem, $\EE[\frac{t_c - t_d}{c_t - d_t} (\frac{d_t}{t_d})^u] \to n^{1-u}$ for $t \to \infty$. We can use Cauchy-Schartz inequality and write
\begin{align}
    \label{eq:expected_Lt2}
    \EE[L_{t, 2}^{(0)}] = \EE[\frac{t_c - t_d}{{(t_d)}^u}]
    \leq \EE[\frac{c_t - d_t}{{(d_t)}^u}] \EE[\frac{t_c - t_d}{c_t - d_t} (\frac{d_t}{t_d})^u]
    \leq C' \EE[\frac{c_t - d_t}{{(d_t)}^u}] \leq C\,C' \EE[{(c_t)}^{2(u-1) - \eta}],
\end{align}
where $C'$ is a constant independant of $t$. Then
\begin{equation}
    \EE[L_t] = \EE\big[\EE[L_t^{(i_t)} | {(i_s)}_s]\big] = \EE\big[\EE[L_t^{(0)} | {(i_s)}_s]\big]
    = \EE[L_t^{(0)}] \leq 2\sqrt{k} D \EE[L_{t,1}^{(0)}] + \EE[L_{t,2}^{(0)}] \in \mathcal{O}(({c_t)}^{2(u-1) - \eta}).
\end{equation}
Combined with~\eqref{eq:expected_Rt}, this shows that
 $\EE[{\Vert \bbeta_t - \bbeta_t^{\star} \Vert}_2] \in \mathcal{O}(({c_t)}^{2(u-1) - \eta})$. As $c_t$ follows a binomial distribution of parameter $(t, \frac{1}{n})$,
 $\frac{c_t}{t} \to \frac{1}{n}$ almost surely when $t \to 0$. Therefore
$\EE[(\frac{c_t}{t})^{2(u-1) - \eta})] \to n^{\eta - 2(u-1)}$, and from Cauchy-Schwartz inequality,
\begin{equation}
    \EE[{\Vert \bbeta_t - \bbeta_t^{\star} \Vert}_2] \leq C \EE[(\frac{c_t}{t})^{2(u-1) - \eta})] t^{2(u-1) - \eta} \in
    \bigO(t^{2(u-1) - \eta}).
\end{equation}
We have reused the fact that converging sequences are bounded. This is enough to conclude.